\documentclass{article}

\usepackage{microtype}
\usepackage{graphicx}
\usepackage{subfigure}
\usepackage{booktabs} 

\usepackage{hyperref}

\usepackage[accepted]{icml2021}

\usepackage{amsmath,amssymb,amsthm}
\usepackage{bm}
\usepackage[capitalize]{cleveref}
\usepackage[nolist,nohyperlinks]{acronym}
\usepackage{nicefrac}
\theoremstyle{plain}
\newtheorem{theorem}{Theorem}[section]
\newtheorem{proposition}[theorem]{Proposition}
\newtheorem{lemma}[theorem]{Lemma}
\newtheorem{corollary}[theorem]{Corollary}
\theoremstyle{definition}

\theoremstyle{remark}

\newcommand{\E}{\mathbb{E}}

\newcommand{\one}[1]{\mathbb{I}\{#1\}}

\newcommand{\norm}[1]{\|#1\|}
\newcommand{\R}{\mathbb{R}}

\newcommand{\cD}{\mathcal{D}}
\newcommand{\cE}{\mathcal{E}}

\newcommand{\cL}{\mathcal{L}}
\newcommand{\cN}{\mathcal{N}}

\newcommand{\mc}[1]{\mathcal{#1}}
\newcommand{\V}{\mathbb{V}}
\renewcommand{\b}[1]{\bm{\mathbf{#1}}}
\newcommand{\const}{\mathrm{const}}
\newcommand{\Cov}{\mathrm{Cov}}
\newcommand{\D}{\mathrm{D}_{\mathrm{KL}}}
\newcommand{\tr}{\mathrm{tr}}

\renewcommand{\epsilon}{\varepsilon}

\DeclareMathOperator*{\argmin}{arg\ min}
\DeclareMathOperator*{\argmax}{arg\ max}

\theoremstyle{plain}

\def\sP{{\mathbb{P}}}
\def\sQ{{\mathbb{Q}}}

\newcommand{\wh}[1]{\hat{#1}}
\newcommand{\tp}{^{\top}}

\newcommand{\bzero}{\boldsymbol{0}}

\newcommand{\ba}{\boldsymbol{a}}

\newcommand{\bA}{\boldsymbol{A}}
\newcommand{\bB}{\boldsymbol{B}}
\newcommand{\bC}{\boldsymbol{C}}
\newcommand{\bD}{\boldsymbol{D}}
\newcommand{\bM}{\boldsymbol{M}}
\newcommand{\bb}{\boldsymbol{b}}

\newcommand{\bhB}{\boldsymbol{\hat{B}}}
\newcommand{\bx}{\boldsymbol{x}}
\newcommand{\bX}{\boldsymbol{X}}
\newcommand{\by}{\boldsymbol{y}}
\newcommand{\bv}{\boldsymbol{v}}
\newcommand{\bY}{\boldsymbol{Y}}

\newcommand{\bF}{\boldsymbol{F}}
\newcommand{\bK}{\boldsymbol{K}}
\newcommand{\bU}{\boldsymbol{U}}
\newcommand{\bP}{\boldsymbol{P}}
\newcommand{\bu}{\boldsymbol{u}}
\newcommand{\bV}{\boldsymbol{V}}

\newcommand{\bI}{\boldsymbol{I}}
\newcommand{\bq}{\boldsymbol{q}}

\newcommand{\bsT}{\boldsymbol{\mathcal{T}}}
\newcommand{\bshT}{\wh{\boldsymbol{\mathcal{T}}}}

\newcommand{\eps}{\varepsilon}
\newcommand{\ve}{\varepsilon}
\newcommand{\beps}{\boldsymbol{\varepsilon}}
\newcommand{\balpha}{\boldsymbol{\alpha}}
\newcommand{\bhalpha}{\wh{\boldsymbol{\alpha}}}

\newcommand{\bxi}{\boldsymbol{\xi}}
\newcommand{\bmu}{\boldsymbol{\mu}}
\newcommand{\bhmu}{\wh{\boldsymbol{\mu}}}

\newcommand{\btheta}{\boldsymbol{\theta}}
\newcommand{\bhtheta}{\wh{\boldsymbol{\theta}}}
\newcommand{\bvartheta}{\boldsymbol{\vartheta}}
\newcommand{\bTheta}{\boldsymbol{\Theta}}
\newcommand{\bPsi}{\boldsymbol{\Psi}}
\newcommand{\bSigma}{\boldsymbol{\Sigma}}
\newcommand{\bhSigma}{\wh{\boldsymbol{\Sigma}}}
\newcommand{\bGamma}{\boldsymbol{\Gamma}}

\newcommand{\bLambda}{\boldsymbol{\Lambda}}

\newcommand{\pr}[1]{\left( #1 \right)}
\newcommand{\br}[1]{\left[ #1 \right]}
\newcommand{\cbr}[1]{\left\{ #1 \right\}}

\newcommand{\bmid}{\;\middle|\;}
\newcommand*\diff{\mathop{}\!\mathrm{d}}

\newcommand{\ind}[1]{\mathbb{I}\left\{ #1 \right\}}

\newcommand{\deln}{^{\backslash n}}
\newcommand{\distiid}{\mathbin{\overset{\mathrm{iid}}{\sim}}}

\DeclareMathOperator{\supp}{\text{supp}}
\DeclareMathOperator{\spn}{span}

\newcommand{\hL}{\wh{\cL}}
\newcommand{\MLE}{^{\text{\scshape{mle}}}}

\newcommand{\brls}{^{\text{\scshape{wbrls}}}}

\newcommand{\gauss}{^{\text{\scshape{g}}}}
\newcommand{\pgauss}{p\gauss}

\renewcommand{\algorithmicrequire}{\textbf{Input:}}
\renewcommand{\algorithmicensure}{\textbf{Output:}}
\renewcommand{\algorithmiccomment}[1]{
\leavevmode\hfill\makebox{$\triangleright$ #1}}
\begin{acronym}
\acro{RLS}{Regularized Least Squares}
\acro{BRLS}[WBRLS]{Weighted Biased Regularized Least Squares}
\acro{ERM}{Empirical Risk Minimization}
\acro{RKHS}{Reproducing kernel Hilbert space}
\acro{DA}{Domain Adaptation}
\acro{PSD}{Positive Semi-Definite}
\acro{SVD}{Singular Value Decomposition}
\acro{GD}{Gradient Descent}
\acro{SGD}{Stochastic Gradient Descent}
\acro{OGD}{Online Gradient Descent}
\acro{SGLD}{Stochastic Gradient Langevin Dynamics}
\acro{MGF}{Moment-Generating Function}
\acro{KL}{Kullback-Liebler}
\acro{EM}{Expectation-Maximization}
\acro{MLE}{Maximum Likelihood Estimator}
\acro{PAC}{Probably Approximately Correct}
\acro{OLS}{Ordinary Least Squares}
\acro{MoM}{Method of Moments}
\end{acronym}

\graphicspath{{graphics/}}

\usepackage[colorinlistoftodos, shadow, color=blue!30!white
					,disable
]{todonotes}

 \newcommand{\todoi}[2][]{\todo[size=\scriptsize,color=red!20!white,#1]{Ilja: #2}}
 \newcommand{\todoc}[2][]{\todo[size=\scriptsize,color=blue!20!white,#1]{Csaba: #2}}

\icmltitlerunning{A Distribution-Dependent Analysis of Meta Learning}

\begin{document}

\twocolumn[
\icmltitle{A Distribution-Dependent Analysis of Meta-Learning}



\icmlsetsymbol{equal}{*}

\begin{icmlauthorlist}
\icmlauthor{Mikhail Konobeev}{uofa}
\icmlauthor{Ilja Kuzborskij}{dm}
\icmlauthor{Csaba Szepesv\'ari}{uofa,dm}
\end{icmlauthorlist}

\icmlaffiliation{uofa}{Computing Science Department, University of Alberta,
Edmonton, Alberta, Canada}
\icmlaffiliation{dm}{DeepMind, London, United Kingdom}

\icmlcorrespondingauthor{Mikhail Konobeev}{konobeev.michael@gmail.com}

\icmlkeywords{Machine Learning, ICML}

\vskip 0.3in
]



\printAffiliationsAndNotice{}  

%
\begin{abstract}

A key problem in the theory of meta-learning is to understand
how the task distributions influence 
transfer risk, the expected error of a meta-learner on a new task
drawn from the unknown task distribution. 
In this paper, focusing on fixed design linear regression with Gaussian noise and
a Gaussian task (or parameter) distribution, we give 
distribution-dependent lower bounds on the transfer risk of \emph{any}
algorithm, while we also show that a \emph{novel, weighted} version
of the so-called biased regularized regression method is able to match these lower bounds up to a fixed constant factor. Notably, the weighting is derived from the covariance of the Gaussian task distribution.
Altogether, our results provide a precise characterization of the difficulty of meta-learning in this Gaussian setting. 
While this problem setting may appear simple, we show that it is rich enough to unify the ``parameter sharing'' and ``representation learning'' streams of meta-learning; in particular, representation learning is obtained as the special case when the covariance matrix of the task distribution is unknown.
For this case
we propose to adopt the EM method, which is shown to enjoy efficient updates in our case. 
The paper is completed by an empirical study of EM. In particular, our experimental results show that 
the EM algorithm can attain the lower bound as the number of tasks grows,
while the algorithm is also successful in competing with its alternatives when used in a representation learning context. 
\end{abstract}

\section{Introduction}
In meta-learning, a learner uses data from past tasks in an attempt to speed up learning on future tasks.
Whether a speedup is possible depends on whether the new task is ``similar'' to the previous ones.
In the formal framework of statistical \emph{meta-learning}
of  \citet{baxter2000model},
the learner is given a sequence of training ``sets''.
The data in each set is independently sampled from an unknown distribution specific to the set, or task,
while each such \emph{task distribution}
is independently sampled from an unknown
meta-distribution, which we shall just call the \emph{environment}. \todoc{do we use environment at all later?}
The learner's \emph{transfer risk} then is its expected prediction loss
on a \emph{target} task freshly sampled from the environment.
\emph{
Can a learner achieve smaller transfer risk by using data from the possibly unrelated tasks? What are the limits of reducing transfer risk?} 



%
%
As an instructive example, consider a popular approach where each of the $n$ tasks is associated with ground truth parameters $\btheta_i \in \R^d$, each of which is assumed to lie close to an unknown vector $\balpha$ that characterizes the environment.
To estimate the unknown parameter vector of the last task, one possibility is to employ
\emph{biased regularization}~\citep{yang2007cross,kuzborskij2013stability,pentina2014pac}, that is, solve the optimization problem
\begin{align*}
  \min_{\btheta} \cbr{\hL_n(\btheta)  + \frac{\lambda}{2} \|\btheta - \bhalpha\|^2}
\end{align*}
where $\hL_n(\cdot)$ is the empirical loss on the $n$th task,
$\lambda>0$ is a regularization parameter that governs the strength of the regularization term
that biases the solution towards
$\bhalpha$, an estimate of $\balpha$. Here, $\bhalpha$ could be obtained,
for example, by averaging parameters estimated on previous tasks~\citep{denevi2018learning}.
This procedure implements the maxim
``learn on a new task, but stay close to what is already learned'',
			which is the basis of many successful meta-learning algorithms, including the above,
			and MAML~\cite{finn2017model}.
                        \todoc{I removed the catastrophic forgetting reference as I could not make out why we are talking about it. Somehow it was too unconnected. We can put it back, but then we should say something substantial about it.}
\todoi{They use a biased regularization, but with a very different metric. Probably we don't have space to discuss this in detail here.}

Early theoretical work in the area focused on studying the
\emph{generalization gap}, which is the difference between the transfer risk and its empirical counterpart. 
\citet{maurer2005algorithmic} gives an upper bound on this gap for a concrete algorithm
which is similar to the biased regularization approach discussed above.
While these bounds are reassuring, they need further work to fully quantify the benefit of meta-learning,
i.e., the gap between the risk of a standard (non-meta) learner and the transfer risk of a meta-learner.
Numerous other works have shown bounds on the generalization gap when using biased regularization, in one-shot learning~\citep{kuzborskij2016fast}, meta-learning~\citep{pentina2014pac}, and sequential learning of tasks~\citep{denevi2018learning,denevi2019learning,balcan2019provable,khodak2019adaptive,finn2019online}.
While some of these works introduced a dependence on the environment distribution, or on the ``regularity'' of the sequence of tasks as appropriate, \emph{they still leave open the question whether the shown dependence is best possible}.  \todoc{true? also, what dependence? how do these results compare to what we do here? any of these bounds will match our lower bound? for $\bSigma=I$?}
%
%

In summary, the main weakness of the cited literature is the \emph{lack of (problem dependent) lower bounds}:
To be able to separate ``good'' meta-learning methods from ``poor'' ones, one needs to know the best achievable performance in a given problem setting.
In learning theory, the most often used lower bounds are \emph{distribution-free} or \emph{problem independent}.
In the context of meta learning, the distribution refers to the distribution over the tasks, or the environment.
The major limitation of a distribution-free approach is that if the class of environments is sufficiently rich, all that the bound will tell us is
that the best standard learner
will have  similar performance to that of
the best meta-learner since the worst-case environment will be one where the tasks are completely unrelated.
As an example, for a linear regression setting with $d$-dimensional parameter vectors,
\citet{lucas2020theoretical} gives the worst-case lower bound
$\Omega(d/ ((2r)^{-d} M + m) )$ for parameter identification where the error is measured in the squared Euclidean distance.
Here, $M$ is the total number of data points in the identically-sized training sets, $m$ is the number of data points in the training set of the target task, and $r\ge 1$ is the radius of the ball that contains the parameter vectors.%
\footnote{This result is stated in Theorem 5 in their paper and
the setting is meta linear regression. For readability, we dropped some constants, such as label noise variance and slightly generalized the cited result by introducing $r$,
which is taken to be $r=1$ in their paper. Indeed, the analysis in the paper is not hard to modify to get the dependence shown on $r$.}
It follows that as $r\to\infty$, the lower bound reduces to that of linear regression and we see that any method that ignores the tasks is competitive with the best meta-learning method.
The pioneering works of \citet{maurer2009transfer,maurer2016benefit} avoid this pathology by introducing empirical quantities that capture task relatedness in the context of linear regression with a common low-dimensional representation.

The bounds can be refined and the pathological limit can be avoided by restricting the set of environments.
This approach is taken by \citet{du2020few} and \citet{tripuraneni2020provable}
who also consider linear regression where the tasks share a common low-dimensional representation.
Their main results show that natural algorithms can take advantage of this extra structure.
In addition, \citet{tripuraneni2020provable} also shows a lower bound on the transfer risk which is
matched by their method
up to logarithmic factors and problem dependent ``conditioning'' constants.

\begin{figure*}
\centering
\includegraphics[width=0.31\linewidth]{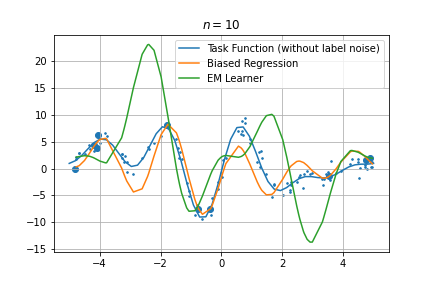}
\includegraphics[width=0.31\linewidth]{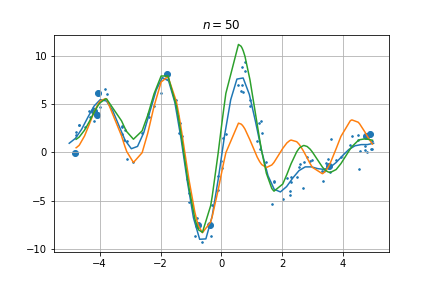}
\includegraphics[width=0.31\linewidth]{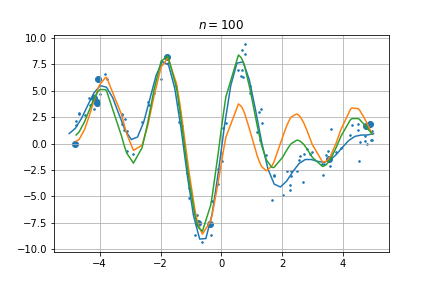}
\caption{Examples of predictions on the synthetic,  `Fourier' meta-learning problem.
Training data is shown in bold, small dots show test data.
We also show the predictions for two learners (at every input) and the target function.
The column correspond to outputs
obtained training on $n\in\{10,50,100\}$ tasks.
Our new algorithm, EM learner, performs quite well.
}
\label{fig:fourier-examples}
\end{figure*}

\paragraph{Our contributions.}
In the present paper we revisit the framework underlying biased regularized regression.
In particular, we propose to study the case
when the unknown parameter vectors for the tasks are generated from a normal distribution with some mean and covariance matrix.
First, we consider the case when the mean is unknown while the covariance matrix is known.
For this case, in the context of fixed-design
linear regression, we prove distribution-dependent lower and upper bounds, which essentially match each other.
The lower bound is a direct lower limit on the transfer risk of \emph{any} meta-learning method.
The upper bound is proven for a version of a \emph{weighted} biased regularized least-squares regression.
Here, the parameters are biased towards the maximum likelihood estimate of the unknown common mean of the task parameter vectors,
and the weighting is done with respect to the inverse covariance matrix of the distribution over the task parameter vectors.
We show that the maximum likelihood estimator can be efficiently computed, which implies that the entire procedure is efficient.
As opposed to the work of \citet{tripuraneni2020provable}, the gap between the lower and upper bounds is a universal constant, regardless of the other parameters of the meta-learning task.
The matching lower and upper bounds together provide a \emph{precise and fine-grained characterization of the benefits of meta-learning}.
Our algorithm shows how one should combine datasets of different cardinalities and suggest specific ways of tuning biased regularized regression based on the noise characteristics of the data and the task structure.
Our lower bounds are based on a rigorously proven novel observation, which may be of interest on its own. According to this observation, any predictor can be treated as a plug-in method that first estimates the unknown task distribution parameters. Hence, to prove a lower bound for the transfer risk, it suffices to do so for plug-in estimators.

In the last part of the paper we consider the case when the covariance matrix
of the task parameter vector distribution is unknown.
Importantly, this case can be seen as a way of \emph{unifying} the representation learning stream of meta-learning with the parameter sharing stream.
In particular, if the covariance matrix is such that $d-s$ of its eigenvalues tend to zero, while the other eigenvalues $s$ are allowed to take on arbitrarily large values, the problem becomes essentially the same as the representation learning problem of
\citet{du2020few,tripuraneni2020provable}.

While we provide no theoretical analysis for this case,
we give a detailed description of how the \ac{EM} algorithm can be used to tackle this problem.
In particular, we show that in this special case the \ac{EM} algorithm
enjoys an efficient \emph{implementation}:
we show how to implement the iterative steps in the loop of the \ac{EM} algorithm
in an efficient way. The steps of this algorithm are
given as closed-form expressions, which are both
intuitive and straightforward to implement.
We demonstrate the effectiveness of the resulting procedure
on a number of synthetic and real benchmarks; \cref{fig:fourier-examples} shows an example on a synthetic benchmark problem, comparing our EM algorithm with the earlier cited (unweighted) ``biased regression'' procedure.
As can be seen from the figure, the EM based learner is significantly more effective.
Further experiments
suggest that the EM learner is almost as effective as the optimal biased weighted regularized regression procedure that is given the unknown parameters.
We found that the EM learner is also competitive as a representation learning algorithm by comparing it to  the algorithm of \citet{tripuraneni2020provable} that is based on the
``method-of-moments'' technique.

\todoc{I removed the alternating optimization comment because I don't see how it helps anyone.}




\section{Setup} 
In the statistical approach to \emph{meta-learning}~\citep{baxter1998theoretical,baxter2000model} 
the learner observes a sequence of \emph{training tuples} $\cD = (D_i)_{i=1}^n$, distributed according to a
\emph{random} sequence of \emph{task} distributions $(P_i)_{i=1}^n$, i.e.\ $D_i \sim P_i$, and furthermore task distributions
are sampled independently from each other from a fixed and unknown \emph{environment} distribution $\mc{P}$.
The focus of this paper is \emph{linear regression with a fixed design} and therefore
each training tuple 
 $D_i = \big((\bx_{i,1}, Y_{i,1}), \ldots, (\bx_{i,m_i}, Y_{i,m_i})\big)$
consists of $m_i$ fixed training \emph{inputs} from $\R^d$ and corresponding random, real-valued \emph{targets} satisfying
\begin{align}
  &Y_{i,j} = \btheta_i\tp \bx_{i,j} + \eps_{i,j}, \label{eq:model}\\
  \text{where} \quad &\eps_{i,j} \distiid \cN(0, \sigma^2), \ \btheta_i \distiid \cN(\balpha, \bSigma) \,,\nonumber
\end{align}
while $(\eps_{i,j})_{i,j}$ and $(\btheta_i)_i$ are also independent \todoc{right?}
from each other.%
\footnote{Technically, $P_i$ consists of Dirac deltas on inputs.} 
A meta-learning environment in this setting is thus given by $\balpha$ and the noise parameters $(\sigma^2, \bSigma)$.
Initially, we will assume that $(\sigma^2, \bSigma)$ is known, while $\balpha$ (just like $(\btheta_i)_i$) is unknown.
The learner observes $\cD$ and needs to produce a prediction of the value 
\begin{align*}
Y = \btheta_n\tp \bx + \eps~
\end{align*}
where $\eps \sim \cN(0,\sigma^2)$ and where $\bx\in \R^d$ is a fixed (non-random) point. 
Our theoretical results will trivially
extend to the case when the learner needs to produce predictions for a sequence of input points
or a fixed distribution over these,
as often considered in meta-learning literature~\citep{denevi2018learning,du2020few}.
\todoi{Added some refs connecting our model to the literature as reviewer requested}
The (random) transfer risk of the learner is defined as
\begin{align*}
  \cL(\bx) &= \E\br{(Y - \hat Y)^2 \bmid \cD }\,.
\end{align*}
The setting described above coincides with the standard fixed-design linear regression setup for $n=1$ and $\bSigma \to \bzero$, for which the behavior of risk is well understood. \todoc{citation}
In contrast, the question that meta-learning poses is whether having $n > 1$, one can design a predictor that achieves lower risk
compared to the approach that only uses the target data. 
Naturally, this is of a particular interest in the small sample regime when for all tasks, $m_i \ll n$, that is when facing scarcity of the training data but having many tasks.
Broadly speaking, this reduces to understanding the behavior of the risk in terms of the interaction between the number of tasks $n$, their sample sizes $(m_1, \ldots, m_n)$, and the task structure given by the noise parametrization $(\sigma^2, \bSigma)$.

\section{Sufficiency of Meta-mean Prediction}
%
In this section we show that there is no loss of generality in considering ``plug-in'' predictors that predict first the unknown meta-mean $\balpha$. 
We also show that biased regularized least-squares estimator belongs to this family.
We start with some general remarks and notation. 

Throughout the rest of the paper, for real symmetric matrices $\bA$ and $\bB$, we use $\bA \succeq \bB$ to indicate that the matrix $\bA - \bB$ is \ac{PSD}.
For $\bx \in \R^d$ and \ac{PSD} matrix $\bA$, 
we let $\|\bx\|_{\bA} = \sqrt{\bx\tp \bA \bx}$. \todoc{I removed from here that $\bx\ne \b0$. We don't quite use this. Only once. So let's add it there.} \todoc{We may want to add our notation concerning Gaussians here.} We use $\norm{\bx}$ to denote the $2$-norm of $\bx$.
%
In the following we will use matrix notation aggregating inputs, targets, and parameters over multiple tasks.
In particular, let the cumulative sample size of all tasks be $M = m_1 + \dots + m_n$
and introduce aggregates for inputs and targets as follows:
\begin{align*}
  &\bX_i =
  \underbrace{
  \begin{bmatrix}
    \bx_{i,1}\tp\\
    \vdots\\
    \bx_{i,m_i}\tp
  \end{bmatrix}
  }_{m_i \times d},
  \bPsi
  =
  \underbrace{
  \begin{bmatrix}
    \bX_1\\
    \vdots\\
    \bX_n
  \end{bmatrix}
  }_{M \times d},
  \bY_i =
  \underbrace{
  \begin{bmatrix}
    Y_{i,1}\\
    \vdots\\
    Y_{i,m_i}
  \end{bmatrix}
  }_{m_i \times 1},
  \bY =
  \underbrace{
  \begin{bmatrix}
    \bY_1\\
    \vdots\\
    \bY_n
  \end{bmatrix}
  }_{M \times 1}\\
  &\bX =
  \underbrace{
\begin{bmatrix}
    \bX_1  & \dots & \bzero\\ 
    \vdots & \ddots & \vdots\\
    \bzero  & \dots & \bX_n
  \end{bmatrix}               
  }_{M \times nd},
  \quad
  \bTheta =
  \underbrace{
  \begin{bmatrix}
    \btheta_1\\
    \vdots\\
    \btheta_n
  \end{bmatrix}
  }_{nd \times 1}~.  
\end{align*}
The matrix representation allows us to compactly state the regression model
simultaneously over all tasks.
In particular, for the $M$-dimensional noise
vector $\beps \sim \cN(\bzero, \sigma^2 \bI)$: 
\begin{equation}
  \label{eq:N_alpha_equiv}
  \bY = \bX \bTheta + \beps
  \quad \Leftrightarrow \quad
  \bY \sim \cN(\bPsi \balpha, \bK)
\end{equation}
where $\balpha$ is a meta-mean of model~\eqref{eq:model} and $\bK$ is
the \emph{marginal} covariance matrix defined as $\bK = \bX(\bI\otimes\bSigma)\bX\tp + \sigma^2\bI$ where $\otimes$ stands for the Kronecker product.
Note that the above equivalence comes from a straightforward observation that
a linear map $\b{X}_i$ applied to the Gaussian r.v.\ $\btheta_i$ is itself
Gaussian with mean
$\E[\bY_i] = \bX_i \E[\btheta_i] = \bX_i \balpha$
and covariance $\bX_i\bSigma\bX_i^T + \sigma^2\b{I}$ which follows
from the property that for any random vector $\bxi$ with covariance matrix $\bC$, and matrix $\bA$ of appropriate dimensions, covariance matrix of $\bA \bxi$ is $\bA \bC \bA\tp$, ultimately giving \cref{eq:N_alpha_equiv}.
%
\subsection{Plug-In Predictors and their Sufficiency}
Both our lower and upper bounds will be derived from analyzing a family of ``plug-in'' 
predictors that aim to estimate $\btheta_n$ through estimating $\balpha$.
As we shall see, weighted biased regularization is also member of this family.

The said family is motivated by applying the well-known bias-variance decomposition to the risk of an arbitrary predictor $A : \supp(P_1) \times \dots \times \supp(P_n) \times \R^d \to \R$. Namely,
\begin{align*}
  \mc{L}(\bx)
    &=
    \E\br{(Y - A(\cD, \bx))^2 \bmid \cD }\\
  &=
    \E\br{
    \pr{\E[Y \,|\, \mc{D}] - A(\cD, \bx)}^2
    +
    \V[Y \,|\, \mc{D}]
    \bmid \cD 
    }
\end{align*}
where we used the law of total expectation and the fact that for any r.v.\ $\xi$, $\E[\xi^2] =
\E[\xi]^2 + \V[\xi]$.
Since the variance term does not depend on $A$, it follows that the prediction problem reduces to predicting the posterior mean $\E[Y \,|\, \mc{D}]$, which,
in our setting, can be given in closed form:
\begin{proposition}
\label{cor:E_cond_D}
  Let $Y = \btheta_n\tp \bx + \eps$ for $\eps \sim \cN(0, \sigma^2)$ and some $\bx \in \R^d$.
  Then,
  \begin{align*}
    \E[Y \,|\, \mc{D}] &= \bx\tp \bsT \pr{\bSigma^{-1}\balpha + \frac{1}{\sigma^2}\bX_n\tp\bY_n}\\
    \text{where} \quad \bsT &= \pr{
           \bSigma^{-1}
           +
           \frac{1}{\sigma^2}\bX_n\tp\bX_n}^{-1}~.
  \end{align*}
\end{proposition}
\begin{proof}
  See~\cref{appendix:posterior-params}.
\end{proof}
Since the only unknown parameter here is the meta-mean $\balpha$, we expect that good predictors will just estimate the meta-mean and use the above formula. That is, these predictors take the form
$(\cD,\bx) \mapsto \bx\tp \bhtheta_n(\balpha(\cD,\bx))$, where
\begin{align}
  \bhtheta_n(\ba) &= \bsT\pr{\bSigma^{-1}\ba + \frac{1}{\sigma^2}\bX_n\tp \bY_n} \quad \ba \in \R^d~, \label{eq:theta-cov}
\end{align}
giving our family of  \emph{plug-in predictors}.
In fact, 
\emph{there is no loss in generality by considering only predictors of the above form.}
Indeed, given some predictor $A$ and $\bx \ne \b0$, 
we can solve $A(\cD,\bx) = \bx\tp \bhtheta_n(\balpha)$ for $\balpha$.
One solution is given by 
$\balpha(\cD,\bx) =  \bSigma \bsT^{-1} \bx c$ where $c=\tfrac{1}{\norm{\bx}^2}\pr{A(\cD,\bx)
  - \sigma^{-2}\bx\tp\bsT\bX_n\tp\bY_n}$.
%
%
%
Hence, to prove a lower bound for any regressor $A$, it will be enough to prove it for algorithms that estimate $\balpha$.
%
%

%

One special estimator of $\balpha$ is the \ac{MLE} estimator, and, 
thanks to~\eqref{eq:N_alpha_equiv}, can be obtained via
$\bhalpha\MLE = \argmax_{\ba \in \R^d} \ln \pgauss(\bY; \bPsi \ba, \bK)$,
where $\pgauss(\bx; \bmu, \bSigma) \propto e^{-\tfrac12 \|\bx - \bmu\|_{\bSigma^{-1}}^2}$ is a Guassian PDF.
Some standard calculations give us
\begin{equation}
  \label{eq:alpha-ml}
  \bhalpha\MLE = (\bPsi\tp\bK^{-1}\bPsi)^{-1}\bPsi\tp\bK^{-1} \bY~.
\end{equation}
%
%
\todoc{I removed a sentence here as I saw no point to it.}
\subsection{Weighted Biased Regularization}
\label{sec:br}
Biased regularization
is a popular transfer learning technique which commonly appears in the regularized formulations of the empirical risk minimization problems,
where one aims at minimizing the empirical risk (such as the mean squared error) while forcing the solution to stay close to some bias variable $\bb$.
Here we consider the \emph{\ac{BRLS}} formulation defined w.r.t.\ bias $\bb$ and some \ac{PSD} matrix $\bGamma$:
\begin{align*}
  &\bhtheta_n\brls
    =    
  \argmin_{\btheta \in \R^d}\cbr{  \hL_n(\btheta) + \frac{\lambda}{2} \|\btheta - \bb\|_{\bGamma}^2 }\\
  \text{where} \quad &\hL_n(\btheta) =
  \sum_{j=1}^{m_n}\pr{Y_{n,j} - \btheta\tp\bx_{n,j}}^2~.
\end{align*}
Remarkably, an estimate $\bhtheta_n\brls$ produced by \ac{BRLS} is \emph{equivalent} to estimator $\bhtheta_n(\bhalpha)$ of~\cref{eq:theta-cov} for the choice of $\bb=\bhalpha, \bGamma=\bSigma^{-1}$, and $\lambda=\sigma^2$.
Thus, \ac{BRLS} is a special member of the family chosen in the previous section. \todoc{well, of course. as that family was general. so we should remove this sentence or rewrite it..}

To see the equivalence, owing to the convenient least-squares formulation, we observe that
\begin{align*}
  \bhtheta_n\brls = \pr{\bX_n\tp\bX_n + \lambda \bGamma}^{-1} (\bX_n\tp\bY_n + \lambda\bGamma\bb)~
\end{align*}
and from here the equivalence follows by substitution. 
%
%
%
A natural question commonly arising in such formulations is how to set the bias term $\bb$.
One choice can be $\bb=\bhalpha\MLE$ and in the following we will see that it is an optimal one.

\section{Problem-Dependent Bounds}
\label{sec:instance-dependent-bounds}
%
%
%

We now present our main results, which are essentially matching lower and upper bounds.
The upper bounds concern the parameter estimator that uses the \ac{MLE} estimate of $\balpha$,
while the lower bounds apply to \emph{any} method.
We also present a more precise lower bound that applies to estimators that are built on \emph{unbiased} meta-mean estimators $\bhalpha$. As we shall see that plug-in predictors based on \ac{MLE} will exactly match this lower bound.
The general lower bounds are also quite precise: They differ from this lower bound only by a (relatively small) universal constant.
We also give a high-probability variant of the same general lower bound.
\begin{theorem}
  \label{thm:main}
  Let $\bx \in \R^d$ and consider the linear regression model~\eqref{eq:model}.
  Let $\bhalpha$ be any unbiased estimator of $\balpha$ based on $\cD$. 
  Then the transfer risk $\cL(\bx)$  of the predictor that predicts 
  $\hat Y =\bx\tp\bhtheta_n(\bhalpha)$ satisfies
  \begin{align}
    \E[\cL(\bx)]
    &\geq
      \bx\tp \bM \bx + \bx\tp\bsT\bx + \sigma^2    \label{eq:unbiased-lower-bound}\\
    \text{where} \quad \bM &= \bsT\bSigma^{-1}(\bPsi\tp\bK^{-1}\bPsi)^{-1} \bSigma^{-1} \bsT~.   \nonumber
  \end{align}
  Moreover, for all predictors 
  we have
  \begin{equation}
    \label{eq:lower-bound-all}
    \E[\cL(\bx)] \geq \frac{\bx\tp \bM \bx}{16 \sqrt{e}} + \bx\tp\bsT\bx + \sigma^2.
  \end{equation}
  Finally, for any $\delta \in (0,1)$, with probability at least $1-\delta$ for all predictors 
  we have
  \[
    \cL(\bx)
    \geq
    \frac12\log\pr{\frac{1}{4(1-\delta)}} \bx\tp \bM \bx
     + \bx\tp\bsT\bx + \sigma^2.
  \]
\end{theorem}
\begin{proof}
The main ideas of the proof are in \cref{sec:proof_sketch}, while the complete proof is given 
in \cref{appendix:lower-bound-proofs}.
\end{proof}
Note that the presented bounds are problem-dependent since they depend on a concrete task structure of the environment characterized by $(\bSigma, \sigma^2)$.
While the strength of the above bound is its generality, this generality 
makes the interpretation of the result challenging. 
We return to the interpretation of this result momentarily, after presenting results for the transfer risk 
%
%
for the plug-in method that uses 
the (unbiased) \ac{MLE} meta-mean estimator $\bhalpha\MLE$ defined in~\cref{eq:alpha-ml}. 
%
%
%
%
\begin{theorem}
  \label{thm:upper_bounds}
For the estimator $\bhtheta_n(\bhalpha\MLE)$ and for any $\bx \in \R^d$ we have
\begin{align}
\label{eq:riskid}
  \E[\cL(\bx)] = \bx\tp \bM \bx + \bx\tp\bsT\bx + \sigma^2\,.
\end{align}
Moreover for the same estimator, with probability at least $1 - \delta, \delta \in (0,1)$ we have
\[
  \cL(\bx) \leq 2\log\pr{\frac{2}{\delta}} \bx\tp \bM \bx
  + \bx\tp\bsT\bx + \sigma^2.
\]
\end{theorem}
\begin{proof}
See \cref{appendix:upper-bound-proofs}.
\end{proof}
Note that \cref{eq:riskid} is an equality for the transfer risk and it matches the lower bound available for unbiased estimators.
This result, together with our lower bound shows that {\em (i)} the predictors based on $\bhalpha\MLE$ is \emph{optimal, with matching constant} within the set of predictors that is based on unbiased estimators of $\balpha$. It also follows that {\em (ii)} apart from a constant factor of $16 \sqrt{e}$ of the transfer risk, this predictor is also optimal among \emph{all} predictors.

\subsection{Interpretation of the Results}
\label{sec:special_cases}
%
%
%
The following two corollaries specialize the lower bound of \cref{thm:main} in a way that will make the results more transparent. Note that while we give these simplified expressions for the lower bound (specifically, \cref{eq:lower-bound-all}), these expressions also remain essentially true for the upper bound for the MLE estimator, since these differ only in minor details.
The proofs of both corollaries are given in \cref{appendix:supp_statements}.

%

Both specializations are concerned with the case when 
the inputs are \emph{isotropic}, meaning that the input covariance matrix of task $i$ is $\frac{m_i}{d} \bI$.

In the first result, in addition, we
assume a \emph{spherical} task structure: $\bSigma = \tau^2 \bI$.
Thus, the coordinates of the parameter vectors $\btheta_i$ are uncorrelated and share the same variance $\tau^2$.
\begin{corollary}
  \label{prop:unrolled_lb_Sigma_cases}
  Assume the same as in case of \cref{eq:lower-bound-all}.
  In addition, let $\bSigma = \tau^2 \bI$, suppose that $\bX_i\tp \bX_i = \frac{m_i}{d} \bI$, and let $\|\bx\| = 1$.
  Then,
{\small
\begin{align*}
  \frac{\E[\cL(\bx)]  -\sigma^2}{\sigma^2} 
  \geq  
  \frac{H_{\tau^2}}{16 \sqrt{e}} \cdot \frac{d^2 \sigma^2}{n \pr{\tau^2 m_n + d \sigma^2}^2}
  +
  \frac{d \tau^2}{\tau^2 m_n + d \sigma^2}  \,,
\end{align*}
}%
where $H_z$ is a harmonic mean of a sequence $(z + \tfrac{d \sigma^2}{m_i})_{i=1}^n$.
\end{corollary}

%
%
%

%
%
In the above bound
the first term vanishes as more tasks are added ($n$ growing).
On the other hand, to decrease the second term, $m_n$ needs to increase.
In particular, as $n\to\infty$, we get
\begin{equation}
  \label{eq:meta_lb_asymp}
  \frac{\E[\cL(\bx)]  -\sigma^2}{\sigma^2} \geq  
  \left(\frac{m_n}{d} + \frac{\sigma^2}{\tau^2}\right)^{-1}
\end{equation}
where $\frac{m_n}{d} + \frac{\sigma^2}{\tau^2}$ can be interpreted as an \emph{effective sample size}.
Thus, while having infinitely many previous tasks have the potential to reduce the loss, the size of this effect is fixed and is related to the noise variance ratios. If $\tau^2\to 0$, having infinitely many tasks will allow perfect prediction, but for any $\tau^2>0$, there is a limit on how much the data of previous tasks can help. 
%
Finally, for the case $n=1$ and $\tau^2 = 0$ we recover the standard lower bound for linear setting $\E[\cL(\bx)]  -\sigma^2 = \Omega(d \sigma^2 / m_1)$.

Our next result is concerned with ``representation learning'', which corresponds to the case when $\bSigma$ is a low rank \ac{PSD} matrix.
\begin{corollary}
\label{cor:unrolled_lb_Sigma_cases2}
Let the inputs be isotropic as before and $\norm{\bx}=1$.
Moreover, let $\bSigma$ be a \ac{PSD} matrix of rank $s \leq d$ with eigenvalues $\lambda_1 \geq \ldots \geq \lambda_s > 0$,%
\footnote{When $s<d$, we replace $\bSigma^{-1}$ with its pseudo-inverse $\bSigma^{\dagger}$.}
and suppose that
$\|\bx\|_{\bP_s\tp \bP_s}^2 = s/d$
where $\bP_s = [\bu_1, \ldots, \bu_s]\tp$ and $(\bu_j)_{j=1}^s$ are unit length eigenvectors of $\bSigma$.
Then,
%
{\small
  \begin{align*}
  \frac{\E[\cL(\bx)]  -\sigma^2}{\sigma^2}
  \geq  
  \frac{H_{\lambda_s}}{16 \sqrt{e}} \cdot
  \frac{s d \sigma^2}{n \pr{\lambda_1 m_n + d \sigma^2}^2} +
  \frac{s \lambda_s}{\lambda_s m_n + d \sigma^2}~.
  \end{align*}
}
%
%
\end{corollary}
%
%
Note that the first term on the right-hand side of the last display 
scales with $sd/n$, where $sd$ is the number of parameter in a matrix that would give the low-dimensional representation and the second term scales with $s/m_n$ for $m_n \gg d \sigma^2/\lambda_s$. Somewhat surprisingly (given that here $\bSigma$ is known),
these essentially match the upper bounds due to 
\citet{du2020few,tripuraneni2020provable}, implying that their results are unimprovable.

\subsection{Proof Sketches}
\label{sec:proof_sketch}
%
Our lower and upper bounds on the risk are based on an identity that holds for the transfer risk of plug-in methods. The identity is essentially a bias-variance decomposition.
\begin{lemma}
  \label{lem:raw_lower_bound}
  For $\bhtheta_n(\bhalpha)$ defined in~\cref{eq:theta-cov}, any task mean estimator $\bhalpha$, and any $\bx \in \R^d$ we have
  $
    \E[\mc{L}(\bx)] =
    \E\br{
      \pr{\bx\tp\bsT\bSigma^{-1}(\balpha - \bhalpha)}^2
    }
    + \bx\tp\bsT\bx + \sigma^2.
  $
\end{lemma}
For the proof of this lemma we need the following proposition whose proof is given
in~\cref{appendix:posterior-params}:
\begin{proposition}
\label{cor:E_cond_D}
  Let $Y = \btheta_n\tp \bx + \eps$ for $\eps \sim \cN(0, \sigma^2)$ and some $\bx \in \R^d$.
  Then, $\E[Y \,|\, \cD] = \bx\tp \bsT \pr{\bSigma^{-1}\balpha + \frac{1}{\sigma^2}\bX_n\tp\bY_n}$
  and $\V[Y \,|\, \cD] = \bx\tp \bsT \bx + \sigma^2$.
\end{proposition}
\begin{proof}[Proof of~\cref{lem:raw_lower_bound}]
Using the law of total expectation and that for a r.v.\ $\xi$ we have $\E[\xi^2] = \E[\xi]^2 + \V[\xi]$,
\begin{align*}
  \mc{L}(\bx)
  &=
    \E\br{(Y - \bhtheta_n(\bhalpha)\tp \bx)^2\mid \cD}\\
  &=
    \E\br{
    \pr{\E[Y \,|\, \cD] - \bhtheta_n(\bhalpha)\tp \bx}^2
    +
    \V[Y \,|\, \cD]
    \mid \cD}\\
  &=
    \E\br{
    \pr{\bx\tp \bsT \bSigma^{-1} \pr{\balpha - \bhalpha}}^2
    \mid \cD}
    + \bx\tp \bsT \bx + \sigma^2\,,
\end{align*}
where identities for $\E[Y \,|\, \cD]$ and $\V[Y \,|\, \cD]$ come from \cref{cor:E_cond_D} and identity for $\bhtheta_n(\bhalpha)$ is due to~\eqref{eq:theta-cov}.
\end{proof}
Thus, to establish universal lower bounds we need to lower bound
$
\E\br{\pr{\bx\tp\bsT\bSigma^{-1}(\balpha - \bhalpha)}^2}
$
for any choice of estimator $\bhalpha$, which in combination with \cref{lem:raw_lower_bound} will prove \cref{thm:main}.
Here, relying on the Cram\'er-Rao inequality (\cref{thm:app:cramer-rao}),
 we only prove a lower bound for unbiased estimators, while the general case, whose proof uses Le Cam's method, is left to \cref{sec:app:general_lb}.
%
%
%
\begin{lemma}
  For any unbiased estimator $\bhalpha$ of $\balpha$ in~\cref{eq:N_alpha_equiv} we have
  $
    \E\left[\left(\bx\tp\bsT\bSigma^{-1}
        (\balpha - \bhalpha)^2\right)\right]
    \geq \bx\tp \bM \bx.
  $
\end{lemma}
\begin{proof}
Recall that according to the equivalence~\eqref{eq:N_alpha_equiv}, $\bY \sim \cN(\bPsi \balpha, \bK)$ and the unknown parameter is $\balpha$.
To compute the Fisher information matrix we first observe that
$\nabla_{\balpha} \ln \pgauss\pr{\bY ; \bPsi \balpha, \bK}
    = \bPsi\tp\bK^{-1}(\bY - \b{\Psi\alpha})
$
and
\begin{align*}
  \bF &=
  \E\br{
  \nabla_{\balpha} \ln \pgauss\pr{\bY ; \bPsi \balpha, \bK}
  \nabla_{\balpha} \ln \pgauss\pr{\bY ; \bPsi \balpha, \bK}\tp
  }\\
  &= \bPsi\tp\bK^{-1}\E\br{(\bY - \bPsi\balpha) (\bY - \bPsi\balpha)\tp} \bK^{-1} \bPsi\\
  &= \bPsi\tp\bK^{-1}\b{\Psi}.
\end{align*}
Thus, by the Cram\'er-Rao inequality we have
$
\E\br{(\balpha - \bhalpha) (\balpha - \bhalpha)\tp} \succeq (\bPsi\tp\bK^{-1}\bPsi)^{-1}~.
$
Finally, left-multiplying by $\bx\tp\bsT\bSigma^{-1}$ and right-multiplying the above by $\bSigma^{-1} \bsT \bx$ gives us the statement.
\end{proof}

\section{Learning with Unknown Task Structure}
\label{sec:EM}
%
%
So far we have assumed that parameters $(\sigma^2, \bSigma)$
characterizing the structure of environment are known, which limits the applicability of the predictor (though does not limit the lower bound).
%
Staying within our framework, a natural idea is to 
estimate all the environment parameters $\cE  = (\balpha, \sigma^2, \bSigma)$ by
maximizing the data marginal log-likelihood
\begin{align*}
  J(\cD, \cE') = \ln \int_{\R^{nd}} p(\cD \,|\, \bvartheta) \diff p(\bvartheta \,|\, \cE')
\end{align*}
over $\cE'$, where $p(\cD, \bTheta, \cE)$ stands for the joint distribution in the model~\eqref{eq:model}.
The above problem is non-convex.
As such, we propose to use  \ac{EM} procedure~\citep{dempster1977maximum}, which 
is known to be a reasonable algorithm for similar settings.\footnote{While the marginal distribution is available in analytic form by~\cref{eq:N_alpha_equiv}, we focus on \ac{EM} because,  in preliminary experiments, direct optimization proved to be numerically unstable.}
\ac{EM} can be derived as a procedure that maximizes a lower bound on $J(\cD, \cE')$: \todoc{if pressed on space this textbook explanation can be removed.}
Jensen's inequality gives us that for any probability measure $q$ on $\R^{nd}$,
$
  J(\cD, \cE')
  \geq \int \ln\pr{\frac{p(\bvartheta, \cD \,|\, \cE')}{q(\bvartheta)}} \diff q(\bvartheta)
$.
This is then maximized in $\cE'$ and $q$ in an alternating fashion:
Letting $\wh{\cE}_t$ to be a parameter estimate at step $t$, we maximize the lower bound in $q$ for a fixed $\cE'=\wh{\cE}_t$, and then obtain $\wh{\cE}_{t+1}$ by maximizing the lower bound in $\cE'$ for a fixed previously obtained solution in $q$.
Maximization in $q$ gives us $q(\bvartheta) = p(\bvartheta \,|\, \cD, \wh{\cE}_t)$, while maximization in $\cE'$ yields
\begin{align}
  \label{eq:EM_argmax}
  \wh{\cE}_{t+1} \in \argmax_{\cE'} \int \ln\pr{p(\bvartheta, \cD \,|\, \cE')} \diff p(\bvartheta \,|\, \cD, \wh{\cE}_t)~.
\end{align}
After some calculations (cf. \cref{appendix:m-step}), this gives \cref{alg:EM}.
During the E-step (lines 4-5), \todoc{oops, references do not resolve: \ref{alg:line:E_1} and \ref{alg:line:E_2}}
the
algorithm computes the parameters of the posterior distribution
$\cN(\btheta_i \,|\, \bhmu_{t,i}, \bshT_{t,i})$ relying on $\wh{\cE}_t$,
and during the M-step (lines 7--9)
it estimates $\wh{\cE}_{t+1}$ based on $(\bhmu_{t,i}, \bshT_{t,i})$.
We propose to detect convergence (not shown) by checking the relative difference between successive
parameter values.  

\begin{algorithm}
  \label{alg:EM}
  \caption{\ac{EM} procedure to estimate $(\balpha, \sigma^2, \bSigma)$}
  \begin{algorithmic}[1]
    \REQUIRE{Initial parameter estimates
    $\wh{\cE}_1 = (\bhalpha_1, \wh{\sigma}_1^2, \bhSigma_1)$}
    \ENSURE{Final parameter estimates
    $\wh{\cE}_t = (\bhalpha_t, \wh{\sigma}_t^2, \bhSigma_t)$}
    \STATE $\bshT_{1,i} \gets
    \bzero, \ \bhmu_{1,i} \gets \bzero$ \quad $i \in [n]$
    \REPEAT
    \FOR[\textcolor{blue}{\texttt{E-step}}]{$i=1,\ldots,n$}
    \STATE $\bshT_{t,i} \gets
    \pr{\bhSigma_t^{-1}
    + \wh{\sigma}_t^{-2}\bX_i\tp\bX_i}^{-1}$ 
    																						\label{alg:line:E_1}
    \STATE $\bhmu_{t,i}
    \gets \bshT_{t,i}\pr{\bhSigma_t^{-1}\bhalpha_t
    + \wh{\sigma}_t^{-2}\bX_i\tp\bY_i}$ 
																						    \label{alg:line:E_2}
    \ENDFOR
    \STATE $\bhalpha_t
    \gets \frac1n \sum_{i=1}^n\bhmu_{t,i}$
    \COMMENT{\textcolor{blue}{\texttt{M-step}}} \label{alg:line:M_start}
    \STATE $\bhSigma_t \gets
    \frac1n \sum_{i=1}^n \pr{\bshT_{t,i}
    + (\bhmu_{t,i}-\bhalpha_t)(\bhmu_{t,i}-\bhalpha_t)\tp}$
    \STATE $\wh{\sigma}^2_t \gets
    \frac1n \sum_{i=1}^n \frac{1}{m_i} \pr{\hL_i(\bhmu_{t,i})
    + \tr\pr{\bX_i \bshT_{t,i} \bX_i\tp}}$ \label{alg:line:M_end}
    \STATE $t \gets t + 1$
    \UNTIL{Convergence (see discussion)}
  \end{algorithmic}
  \label{alg:EM}
\end{algorithm}

\section{Experiments}
\label{sec:exps}
\begin{figure*}[t]
\centering
\includegraphics[width=0.45\linewidth]{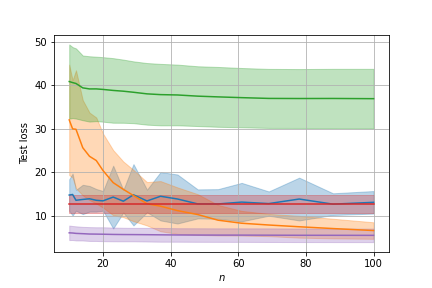}
\includegraphics[width=0.45\linewidth]{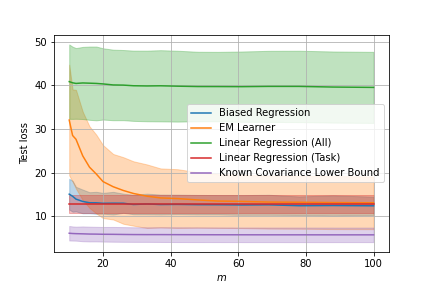}
\caption{Test errors on Fourier synthetic experiment with
changing number of tasks $n$ and number of samples per task $m$.
When one of the parameters changes, the other one is set to 10.}
\label{fig:fourier}
\end{figure*}
In this section we present experiments designed to verify three hypotheses:
\emph{(i)}
Under ideal circumstances,
the predictor $\bx^\top \bhtheta_n(\bhalpha\MLE)$ is superior to its
alternatives, including biased, but unweighted regression;
\emph{(ii)} The \ac{EM}-algorithm reliably recovers unknown parameters of the environment and is also suitable for representation learning; \todoc{figure!}
\emph{(iii)} our distribution-dependent lower bound~\cref{eq:unbiased-lower-bound} is numerically sharp. In addition, we briefly report on experiments with a real-world dataset.
\setlength{\parskip}{0.15\baselineskip}%
\paragraph{Baselines.}
We consider two non-meta-learning baselines, that is \texttt{Linear Regression (All)} --- \ac{OLS} fitted on $\cD\deln = (D_i)_{i=1}^{n-1}$, which excludes the newly observed task, and \texttt{Linear Regression (Task)} --- \ac{OLS} fitted on a newly encountered task $D_n$.
Next, we consider meta-learning algorithms.
We report performance of  the \emph{unweighted} \texttt{Biased Regression}
procedure with
bias set to the least squares solution $(\sum_{i\neq n} \bX_i \tp \bX_i)^{-1}\sum_{i\neq n} \bX_i \tp \bY_i$
and $\lambda$ found by cross-validation (cf. \cref{sec:lambdaopt}).
Note that the bias and the regularization coefficient are found on $\cD\deln$, while $D_n$ is used for the final fitting.
\texttt{EM Learner} is estimator \eqref{eq:theta-cov} with all environment parameters found by \cref{alg:EM} on $\cD\deln$. The convergence threshold was set to $10^{-6}$ while the maximum number of iterations was set to $10^3$.
Finally, we report numerical values of \cref{eq:unbiased-lower-bound} as \texttt{Known Covariance Lower Bound}.
For all of the experiments
we show averages and standard deviations of the mean test errors
computed over $30$ independent runs of that experiment.
\paragraph{Synthetic Experiments.} We conduct synthetic experiments
on datasets with Fourier generated features and features sampled
from a $d$-dimensional unit sphere. In all of the synthetic experiments
we have $\b\alpha=\b0, \sigma^2=1$ and $\b\Sigma$ generated by computing
$
\bSigma = \b{LL}\tp + \eta\bI
$
where $L_{ij} = \ind{i \geq j}Z_{ij}$ with
$Z_{ij}$ and $\eta$ sampled from the standard normal distribution.
Test error is computed on 100 test tasks using 10 examples
for training and $100$ examples for testing.

For the Fourier-features,
we sample a value
$u\sim\mc{U}(-5,5)$
and compute features by evaluating $d=11$ Fourier basis functions at $u$:
$x_j = \one{1\le j\le 5}\sin(\frac{j}5 \pi u)
+\one{6\le j \le 10}\cos(\frac{j-5}5 \pi  u)+ \one{j=11}$, where $\one{E}=1$ if $E$ is true and $\one{E}=0$ otherwise.
Examples of these tasks and results of meta-learning on some of these
were shown in \cref{fig:fourier-examples}.
In \cref{fig:fourier} we show the test errors for various meta-learners
while varying the number
of tasks $n$ and task sizes $m$. For the `spherical' data, the same is shown in \cref{fig:spherical}.
Here, we generate $\bx$ from a $d=42$ dimensional unit sphere. In both
experiments for sufficiently large number of training tasks
the EM-based learner approaches the optimal estimator
even when the number of examples per task is less than the dimensionality
of that task.

In the context of the `Fourier' dataset,
we also experimented with generating low-rank $\bSigma$, corresponding to
the challenge of learning a low-dimensional representation, shared across the tasks.
We found that the EM-based meta learner stays competitive in this setting.
To save space, the results are presented in \cref{sec:apxexps}.
\begin{figure*}
\vspace*{-0.1in}
\centering
\includegraphics[width=0.43\linewidth]{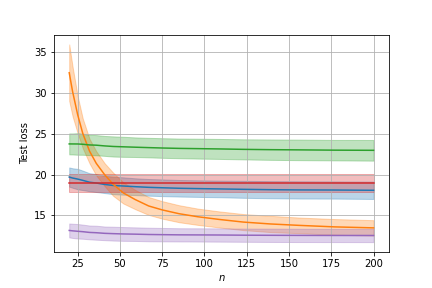}
\includegraphics[width=0.43\linewidth]{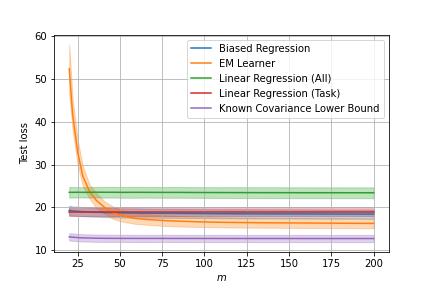}
\vspace*{-0.1in}
\caption{Test error on spherical synthetic experiment with
changing number of tasks $n$ and number of samples per task $m$.
When one of the parameters changes, the other one is set to 40.}
\label{fig:spherical}
\end{figure*}
\paragraph{Real Dataset Experiment.} We also conducted experiments
on a real world dataset containing information about students in
$139$ schools in years $1985$-$1987$~\citep[School Dataset]{dua2017openml}.  We adapt the dataset to
a meta-learning problem with the goal to predict the exam score
of students based on the student-specific and school-specific
features. After one-hot encoding of the categorical values there
are $d=27$ features for each student.
We randomly split schools into two subsets: The first, consisting of $100$ schools, forms $\cD\deln$ (used for training the bias, $\lambda$ selection, and EM).
The second subset consists of $39$ schools, where each school is further split into $80\%/20\%$ for
the final training
and testing of the meta-learners.
Results are given
in \cref{fig:school}.
We can see that while both Biased Regression and EM Learner outperform
regression, their
performance is very similar. This could be attributed to the fact
that the features mostly contain weakly relevant information,
which is confirmed by inspecting the coefficient vector.
\begin{figure}[t]
\vspace*{-0.1in}
\centering
\includegraphics[width=0.45\textwidth]{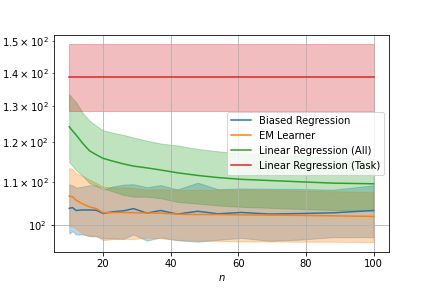}
\vspace*{-0.2in}
\caption{
Test error on the School Dataset. Up to $100$ schools are
used for fitting environment-related parameters (see text for
details) and the remaining $39$ are used as the target task.
}
\label{fig:school}
\end{figure}

\paragraph{Representation learning experiments I.}

\begin{figure*}[t]
\centering
\includegraphics[width=0.45\linewidth]{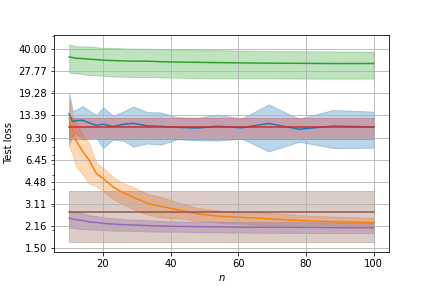}
\includegraphics[width=0.45\linewidth]{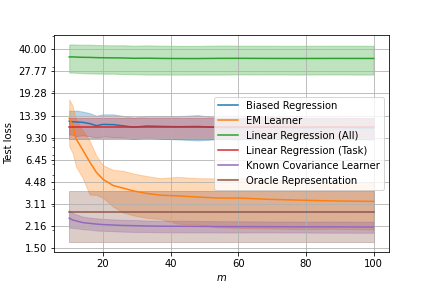}
\caption{Test error when the task covariance matrix is low-rank.
As usual, on the left the number of tasks is changed, on the right, the number of training datapoints (per task). When one parameter is varied, the other is set to the value of $10$.
}
\label{fig:fourier-lowrank}
\end{figure*}

Our next figure (\cref{fig:fourier-lowrank}) shows the outcomes of
experiments for the Fourier task but when $\bSigma$ is low-rank.
As can be seen, the EM based learner excels in exploiting the
low-rank structure. For this experiment we have the same setup as
for the Fourier experiment, but the covariance matrix $\bSigma$ is
generated by computing $\bSigma=\b{LL}\tp$ where $\b{L}$ is a
$d\times r$ matrix with $r=\lfloor d/2\rfloor=5$ and elements
$L_{i,j}\sim\cN(0,1)$. Note that in this case we can write
$\btheta_i = \b{B}\boldsymbol{w}_i$
for some matrix $\b{B}$ of size $d\times r$ and vector
$\b{w}_i$ of size $r$ sampled from multivariate normal distribution.
Thus, if the matrix $\b{B}$ is known or estimated during
training, one can project the features $\bx_{i,j}$ onto a lower-dimensional
space by computing $\b{B}^\top \bx_{i,j}$ to speed up the
adaptation to new tasks by running least-squares regression to estimate
$\boldsymbol{w}_i$ instead of $\b{\theta}_i$.

In addition to the baselines described in the main text, we compared
\texttt{EM Learner} with two additional baselines: one is based on the \ac{MoM} estimator from \citet{tripuraneni2020provable} (not shown
on the figure), and another which we refer to as \texttt{Oracle
Representation}.  We omit displaying the error of the method of
moments estimator since for the features generated as in this
experiment it is not able to perform estimation of the subspace and
leads to test errors with values around $60$. At the same time, as shown on
\cref{fig:fourier-lowrank} (left) we observe that \texttt{EM Learner}
can outperform \texttt{Oracle Representation} which assumes the
knowledge of the covariance matrix $\bSigma$ from which it computes
the subspace matrix $\b{B}$ and uses it to obtain
lower-dimensional representation of the features when adapting to
a new task via least-squares, as described above.
This is possible
because
the coefficients estimated
by EM are biased toward $\balpha$ which does not happen with least
squares regression in the lower dimensional subspace
and this is beneficial, especially when the
number of test-task training examples is small.

\paragraph{Representation learning experiments II.}
\begin{figure}
\centering
\includegraphics[width=0.85\linewidth]{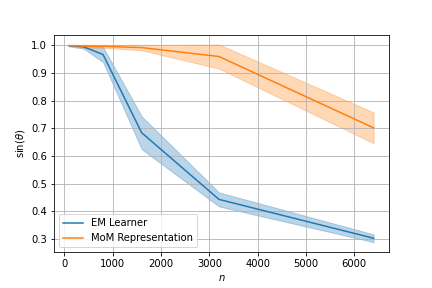}
\caption{Max-correlation $d_{\max}(\bhB, \bB)$
  between the estimated matrix $\bhB$ (by the respective algorithm) and the ground truth matrix $\bB$ while increasing number of tasks $n$.
}
\label{fig:EM_vs_MoM}
\end{figure}

To validate our implementation of the \ac{MoM} estimator of \citet{tripuraneni2020provable} and to investigate more whether \ac{EM} is preferable to the \ac{MoM} estimator beyond the setting that is ideal for the \ac{EM} method we considered the experimental setup of \citet{tripuraneni2020provable}.

To explain the setup, we recall that
the \ac{MoM} estimator computes an estimate $\bhB$ of the ground truth matrix $\bB$.
\citet{tripuraneni2020provable} proves results for the \emph{max-correlation} between $\bhB$ and $\bB$,
and also reports experimentally measured  max-correlation values between the ground truth and the \ac{MoM} computed matrix.
The max-correlation between matrices $\bA$ and $\bA'$ is based on the definition of
\emph{principal angles} and is equal to $d_{\max}(\bA, \bA') = \sqrt{1 - \cos^2(\bA,\bA')} $ where
$\cos(\bA,\bA') =
\max_{\bu \in \spn(\bA):\norm{\bu}=1} \max_{\bv \in \spn(\bA'): \norm{\bv}=1} \bu \tp \bv $.
Intuitively, max-correlation captures how well the subspaces spanned by matrices $\bA$ and $\bA'$ are aligned.

To compare our \ac{EM} estimator to \ac{MoM}
we run the \ac{EM} estimator as described in \cref{alg:EM}, and once the final estimate $\bhSigma$ is obtained, we reduce its rank by clipping eigenvalues $\lambda_{s+1} \geq \ldots \geq \lambda_{d}$ to $0$.

We follow the experimental setup of \citet{tripuraneni2020provable}, that is, inputs are generated as $\bx_i \sim \cN(0, \bI_d)$, while the regression model is given by \cref{eq:model} with $(\sigma^2, \bSigma) = (1, \tfrac{1}{s} \bB\bB\tp )$.
Here, columns of $\bB \in \R^{d \times s}$ are sampled from a uniform distribution on a unit $d$-sphere.
Finally, the number of examples per previously observed task is set as $m_1 = ... = m_{n-1} = 5$, the representation rank is $s = 5$, the input dimension is $d = 100$, and the experiment is repeated $30$ times. Since we only estimate
the subspace matrix we do not use the data from the test task $(\bX_n,\by_n)$.

We report our results in \cref{fig:EM_vs_MoM}, plotting the max-correlation between $\bhB$ found by the respective algorithm and $\bB$, while increasing the number of tasks.
We see that \verb!EM learner! considerably outperforms \verb!MoM Representation! in terms of the subspace estimation to the degree captured by max-correlation.
While we suspect that the improvement is due to the joint optimization over the covariance of environment and the mean of the environment (the bias in biased regularization), the detailed understanding of this effect is left for the future work.

\section{Conclusions}
\label{sec:conc}
%
While ours is the first work to derive matching, distribution-dependent lower and upper bounds, much works remains to be done:  our approach to derive meta-learning algorithms based on a probabilistic model should be applicable more broadly and could lead to further interesting developments in meta-learning.
The most interesting narrower question is to theoretically analyze the EM algorithm. Doing this in the low-rank setting looks particularly interesting. We hope that our paper will inspire other researchers to do further work in this area.

%


\bibliography{ref}

\begin{thebibliography}{22}
\providecommand{\natexlab}[1]{#1}
\providecommand{\url}[1]{\texttt{#1}}
\expandafter\ifx\csname urlstyle\endcsname\relax
  \providecommand{\doi}[1]{doi: #1}\else
  \providecommand{\doi}{doi: \begingroup \urlstyle{rm}\Url}\fi

\bibitem[Baxter(1998)]{baxter1998theoretical}
Baxter, J.
\newblock Theoretical models of learning to learn.
\newblock In S.~Thrun, L.~P. (ed.), \emph{Learning to learn}, pp.\  71--94.
  Springer, 1998.

\bibitem[Baxter(2000)]{baxter2000model}
Baxter, J.
\newblock A model of inductive bias learning.
\newblock \emph{Journal of Artificial Intelligence Research}, 12:\penalty0
  149--198, 2000.

\bibitem[Bretagnolle \& Huber(1979)Bretagnolle and
  Huber]{bretagnolle1979estimation}
Bretagnolle, J. and Huber, C.
\newblock Estimation des densit{\'e}s: risque minimax.
\newblock \emph{Zeitschrift f{\"u}r Wahrscheinlichkeitstheorie und verwandte
  Gebiete}, 47\penalty0 (2):\penalty0 119--137, 1979.

\bibitem[Dempster et~al.(1977)Dempster, Laird, and Rubin]{dempster1977maximum}
Dempster, A.~P., Laird, N.~M., and Rubin, D.~B.
\newblock Maximum likelihood from incomplete data via the em algorithm.
\newblock \emph{Journal of the Royal Statistical Society: Series B
  (Methodological)}, 39\penalty0 (1):\penalty0 1--22, 1977.

\bibitem[Denevi et~al.(2018)Denevi, Ciliberto, Stamos, and
  Pontil]{denevi2018learning}
Denevi, G., Ciliberto, C., Stamos, D., and Pontil, M.
\newblock Learning to learn around a common mean.
\newblock In \emph{Conference on Neural Information Processing Systems
  (NeurIPS)}, pp.\  10169--10179, 2018.

\bibitem[Denevi et~al.(2019)Denevi, Ciliberto, Grazzi, and
  Pontil]{denevi2019learning}
Denevi, G., Ciliberto, C., Grazzi, R., and Pontil, M.
\newblock Learning-to-learn stochastic gradient descent with biased
  regularization.
\newblock In \emph{International Conference on Machine Learing (ICML)}, 2019.

\bibitem[Du et~al.(2020)Du, Hu, Kakade, Lee, and Lei]{du2020few}
Du, S.~S., Hu, W., Kakade, S.~M., Lee, J.~D., and Lei, Q.
\newblock Few-shot learning via learning the representation, provably.
\newblock arXiv:2002.09434, 2020.

\bibitem[Dua \& Graff(2017)Dua and Graff]{dua2017openml}
Dua, D. and Graff, C.
\newblock {UCI} machine learning repository, 2017.
\newblock URL \url{http://archive.ics.uci.edu/ml}.

\bibitem[Finn et~al.(2017)Finn, Abbeel, and Levine]{finn2017model}
Finn, C., Abbeel, P., and Levine, S.
\newblock Model-agnostic meta-learning for fast adaptation of deep networks.
\newblock In \emph{International Conference on Machine Learing (ICML)}, 2017.

\bibitem[Finn et~al.(2019)Finn, Rajeswaran, Kakade, and Levine]{finn2019online}
Finn, C., Rajeswaran, A., Kakade, S., and Levine, S.
\newblock Online meta-learning.
\newblock \emph{International Conference on Machine Learing (ICML)}, 2019.

\bibitem[Khodak et~al.(2019{\natexlab{a}})Khodak, Balcan, and
  Talwalkar]{balcan2019provable}
Khodak, M., Balcan, M.-F., and Talwalkar, A.
\newblock Provable guarantees for gradient-based meta-learning.
\newblock In \emph{International Conference on Machine Learing (ICML)}, pp.\
  424--433, 2019{\natexlab{a}}.

\bibitem[Khodak et~al.(2019{\natexlab{b}})Khodak, Balcan, and
  Talwalkar]{khodak2019adaptive}
Khodak, M., Balcan, M.-F.~F., and Talwalkar, A.~S.
\newblock Adaptive gradient-based meta-learning methods.
\newblock \emph{Advances in Neural Information Processing Systems},
  32:\penalty0 5917--5928, 2019{\natexlab{b}}.

\bibitem[Kuzborskij \& Orabona(2013)Kuzborskij and
  Orabona]{kuzborskij2013stability}
Kuzborskij, I. and Orabona, F.
\newblock Stability and {H}ypothesis {T}ransfer {L}earning.
\newblock In \emph{International Conference on Machine Learing (ICML)}, pp.\
  942--950, 2013.

\bibitem[Kuzborskij \& Orabona(2016)Kuzborskij and Orabona]{kuzborskij2016fast}
Kuzborskij, I. and Orabona, F.
\newblock Fast {R}ates by {T}ransferring from {A}uxiliary {H}ypotheses.
\newblock \emph{Machine Learning}, pp.\  1--25, 2016.
\newblock ISSN 1573-0565.
\newblock \doi{10.1007/s10994-016-5594-4}.

\bibitem[Lattimore \& Szepesv{\'a}ri(2018)Lattimore and
  Szepesv{\'a}ri]{lattimore2018bandit}
Lattimore, T. and Szepesv{\'a}ri, C.
\newblock \emph{Bandit algorithms}.
\newblock Cambridge University Press, 2018.

\bibitem[Lucas et~al.(2021)Lucas, Ren, Kameni, Pitassi, and
  Zemel]{lucas2020theoretical}
Lucas, J., Ren, M., Kameni, I., Pitassi, T., and Zemel, R.
\newblock Theoretical bounds on estimation error for meta-learning.
\newblock In \emph{International Conference on Learning Representations}, 2021.

\bibitem[Maurer(2005)]{maurer2005algorithmic}
Maurer, A.
\newblock Algorithmic stability and meta-learning.
\newblock \emph{Journal of Machine Learning Research}, 6\penalty0
  (Jun):\penalty0 967--994, 2005.

\bibitem[Maurer(2009)]{maurer2009transfer}
Maurer, A.
\newblock Transfer bounds for linear feature learning.
\newblock \emph{Machine Learning}, 75\penalty0 (3):\penalty0 327--350, 2009.

\bibitem[Maurer et~al.(2016)Maurer, Pontil, and
  Romera-Paredes]{maurer2016benefit}
Maurer, A., Pontil, M., and Romera-Paredes, B.
\newblock The benefit of multitask representation learning.
\newblock \emph{Journal of Machine Learning Research}, 2016.

\bibitem[Pentina \& Lampert(2014)Pentina and Lampert]{pentina2014pac}
Pentina, A. and Lampert, C.
\newblock A pac-bayesian bound for lifelong learning.
\newblock In \emph{International Conference on Machine Learing (ICML)}, pp.\
  991--999, 2014.

\bibitem[Tripuraneni et~al.(2020)Tripuraneni, Jin, and
  Jordan]{tripuraneni2020provable}
Tripuraneni, N., Jin, C., and Jordan, M.~I.
\newblock Provable meta-learning of linear representations.
\newblock arXiv:2002.11684, 2020.

\bibitem[Yang et~al.(2007)Yang, Yan, and Hauptmann]{yang2007cross}
Yang, J., Yan, R., and Hauptmann, A.~G.
\newblock Cross-domain video concept detection using adaptive svms.
\newblock In \emph{Proceedings of the 15th ACM international conference on
  Multimedia}, pp.\  188--197, 2007.

\end{thebibliography}
\bibliographystyle{icml2021}

\appendix
\onecolumn


\section{Parameters of the Posterior Distribution}
\label{appendix:posterior-params}
Recall that
\[
  p_n(\btheta_n \,|\, \mc{D}) \propto \pgauss(\bY_n\,|\, \bX_n \btheta_n, \sigma^2 \bI) \pgauss(\btheta_n \,|\, \balpha, \bSigma)~.
\]
We first give a handy proposition for the posterior distribution over the parameters $\btheta_n$.
\begin{proposition}
  \label{prop:posterior_ll}
  \[
    \ln p_n(\btheta_n)
    = -\frac{1}{2}(\btheta_n - \bmu)
    \bsT^{-1}(\btheta_n - \bmu) + \const(\btheta_n)
  \]
  where covariance is
  \[
    \bsT = \left(\bSigma^{-1} + \frac{1}{\sigma^2}\bX_n\tp\bX_n\right)^{-1}
  \]
  and mean is
  \[
    \bmu
    = \bsT\left(\bSigma^{-1}\balpha + \frac{1}{\sigma^2}
      \bX_n\tp\bY_n\right)~.
  \]
\end{proposition}

\begin{proof}
The log-likelihood is the following chain of identities:
\begin{align*}
  \ln p_n(\btheta_n)
  &=
    -\frac{\sum_{j=1}^{m_n}(
    \bx_{n,j}\tp\b{\theta}_n - Y_{n,j})^2}{2\sigma^2}
    -\frac{1}{2}(\btheta_n - \balpha)\tp
    \bSigma^{-1}(\btheta_n - \balpha) + \const(\btheta_n)\\
  &= -\frac{\sum_{j=1}^{m_n}(\btheta_n\tp\bx_{n,j}\bx_{n,j}\tp\btheta_n
    -2Y_{n,j}\bx_{n,j}\tp\b{\theta}_n)}{2\sigma^2}\\
  &= - \frac{1}{2}(\btheta_n\tp\bSigma^{-1}\btheta_n
    - 2\balpha\tp\bSigma^{-1}\btheta_n)
    + \const(\btheta_n)\\
  &= -\frac{1}{2}\left(\btheta_n\tp\bsT^{-1}\btheta_n
    - 2\left(\bSigma^{-1}\balpha
    + \frac{1}{\sigma^2}\bX_n\tp\bY_n\right)\tp\btheta_n
    \right) + \const(\btheta_n)\\
  &= -\frac{1}{2}(\btheta_n - \bmu)
    \bsT^{-1}(\btheta_n - \bmu) + \const(\btheta_n)~.
\end{align*}
\end{proof}
First note that the first consequence of \cref{prop:posterior_ll} is a \ac{MLE} for $\btheta_n$,
\[
  \bhtheta_n\MLE = \bsT\pr{\bSigma^{-1}\balpha + \frac{1}{\sigma^2}\bX_n\tp \bY_n}~.
\]
The second consequence is the following corollary which is obtained by taking
$Y = \btheta_n\tp \bx + \eps$
and simply observing that the mean of a $p_n(\btheta_n \,|\, \mc{D})$ is $\bmu$, and so $\E[Y \,|\, \mc{D}] = \bx\tp \bmu$ while the variance is $\V[Y \,|\, \mc{D}] = \E[(\bx\tp \btheta_n + \eps)^2 \,|\, \mc{D} ] - \E[(\bx\tp \btheta_n + \eps) \,|\, \mc{D} ]^2 = \bx\tp \bsT \bx + \sigma^2$.
\paragraph{\cref{cor:E_cond_D} (restated).}\emph{
  Let $Y = \btheta_n\tp \bx + \eps$ for $\eps \sim \cN(0, \sigma^2)$ and some $\bx \in \R^d$.
  Then,
  \begin{align*}
    &\E[Y \,|\, \mc{D}] = \bx\tp \bsT \pr{\bSigma^{-1}\balpha + \frac{1}{\sigma^2}\bX_n\tp\bY_n}\\
    \text{and} \quad &\V[Y \,|\, \mc{D}] = \bx\tp \bsT \bx + \sigma^2~.
  \end{align*}}

\section{Proof of the Lower Bounds}\label{appendix:lower-bound-proofs}
Our task reduces to establishing lower bounds on
\begin{equation}
  \label{eq:app:E_alpha_halpha_deviation}
  \E\br{\pr{\bx\tp\bsT\bSigma^{-1}(\balpha - \bhalpha)}^2}
\end{equation}
for any choice of estimator $\bhalpha$, which in combination with \cref{lem:raw_lower_bound} will prove \cref{thm:main}.
In the next section we first prove a lower bound for any unbiased estimator relying on the Cramér-Rao inequality.
In what follows, in \cref{sec:app:general_lb}, we will show a general bound in~\cref{lem:E_alpha_general_lb} valid for any estimator (possibly biased) using a \emph{hypothesis testing} technique (see, e.g.\ \citep[Chap. 13]{lattimore2018bandit}).
Finally, in~\cref{lem:P_alpha_general_lb} we prove a high-probability lower bound on~\cref{eq:app:E_alpha_halpha_deviation}.
\subsection{Lower Bound for Unbiased Estimator $\bhalpha$}
\begin{theorem}[Cramér-Rao inequality]
  \label{thm:app:cramer-rao}
Suppose that $\balpha \in \R^d$ is an unknown deterministic parameter with a probability density function $f(x \,|\, \balpha)$ and that $\bhalpha$ is an unbiased estimator of $\balpha$.
Moreover assume that for all $i,j \in [d]$, $x : f(x \,|\, \balpha) > 0$, $\frac{\partial^2}{\partial \alpha_i \partial \alpha_j} \ln f(x \,|\, \balpha)$ exists and is finite, and $\frac{\partial^2}{\partial \alpha_i \partial \alpha_j} \int \bhalpha f(x \,|\, \balpha) \diff x = \int \bhalpha \pr{\frac{\partial^2}{\partial \alpha_i \partial \alpha_j} f(x \,|\, \balpha)} \diff x$.

Then, for the Fisher information matrix defined as
\[
  \bF = - \E\br{\nabla_{\balpha} \ln f(X \,|\, \balpha) \nabla_{\balpha} \ln f(X \,|\, \balpha)\tp}
\]
we have
\[
  \E\br{(\bhalpha - \E[\bhalpha]) (\bhalpha - \E[\bhalpha])\tp}
  \succeq
  \bF^{-1}~.
\]
\end{theorem}

\begin{lemma}
  For any unbiased estimator $\bhalpha$ of $\balpha$ in~\cref{eq:N_alpha_equiv} we have
  \begin{equation}
    \E\left[\left(\bx\tp\bsT\bSigma^{-1}
        (\balpha - \bhalpha)^2\right)\right]
    \geq \bx\tp\bsT\bSigma^{-1}
    (\bPsi\tp\bK\bPsi)^{-1}\bSigma^{-1}\bsT\bx.
  \end{equation}
\end{lemma}
\begin{proof}
Recall that according to the equivalence~\eqref{eq:N_alpha_equiv} $\bY \sim \cN(\bPsi \balpha, \bK)$ and the unknown parameter is $\balpha$.
To compute the Fisher information matrix we first observe that
\begin{align*}\label{eq:grad-log-likelihood}
  &\nabla_{\balpha} \ln \pgauss\pr{\bY ; \bPsi \balpha, \bK}
    = \bPsi\tp\bK^{-1}(\bY - \b{\Psi\alpha})
\end{align*}
and so
\begin{align*}
  \bF &=
  \E\br{
  \nabla_{\balpha} \ln \pgauss\pr{\bY ; \bPsi \balpha, \bK}
  \nabla_{\balpha} \ln \pgauss\pr{\bY ; \bPsi \balpha, \bK}\tp
  }\\
  &= \bPsi\tp\bK^{-1}\E\br{(\bY - \bPsi\balpha) (\bY - \bPsi\balpha)\tp} \bK^{-1} \bPsi\\
  &= \bPsi\tp\bK^{-1}\b{\Psi}.
\end{align*}
Thus, by \cref{thm:app:cramer-rao} we have
\begin{align*}
\E\br{(\balpha - \bhalpha) (\balpha - \bhalpha)\tp} \succeq (\bPsi\tp\bK^{-1}\bPsi)^{-1}~.
\end{align*}
Finally, left-multiplying by $\bx\tp\bsT\bSigma^{-1}$ and right-multiplying the above by $\bSigma^{-1} \bsT \bx$ gives us the statement.
\end{proof}
\subsection{Lower Bound for Any Estimator $\bhalpha$}
\label{sec:app:general_lb}
The proof of is based on the following lemma.
\begin{lemma}[\citealt{bretagnolle1979estimation}]
\label{lemma:bretagnolle-huber}
Let $P$ and $Q$ be probability measures on the same measurable

space $(\Omega, \mathcal{F})$, and let $A\in\mathcal{F}$ be an
arbitrary event. Then,
\begin{equation}
P(A) + Q(A^c) \geq \frac{1}{2}\exp(-\D(P, Q)),
\end{equation}
where $\D(P, Q) = \int_{\Omega} \ln\pr{P(\omega) / Q(\omega)} \diff P(\omega)$ denotes Kullback-Leibler divergence between $P$ and $Q$  and $A^c = \Omega \setminus A$ is the complement of $A$.
\end{lemma}

\begin{lemma}
  \label{lem:E_alpha_general_lb}
  For any estimator $\bhalpha$ of $\balpha$ in~\cref{eq:N_alpha_equiv} we have
  \begin{align*}
    \E\br{\pr{\bx\tp \bsT \bSigma^{-1}(\bhalpha - \balpha)}^2}
    \geq
    \frac{\bx\tp \bM \bx}{16 \sqrt{e}}~.
  \end{align*}
\end{lemma}
\begin{proof}
Throughout the proof let $\bq = \bSigma^{-1} \bsT \bx$.
Consider two meta-learning problems with target distributions $\sP$ and $\sQ$ characterized by two means: $\balpha_{\sP} = \bzero$ and $\balpha_{\sQ} = \Delta (\bPsi\tp\bK^{-1}\bPsi)^{-1} \bq$ where $\Delta > 0$ is a free parameter to be tuned later on.
Thus, according to our established equivalence~\eqref{eq:N_alpha_equiv}, in these two cases targets are generated by respective models $\sP = \cN(\bzero, \bK)$ and $\sQ = \cN(\Delta \bPsi (\bPsi\tp\bK^{-1}\bPsi)^{-1} \bq, \bK)$.

Recall our abbreviation $\bM = \bsT\bSigma^{-1}(\bPsi\tp\bK^{-1}\bPsi)^{-1} \bSigma^{-1} \bsT$.
Markov's inequality gives 
\begin{align*}
  \E_{\sP}\br{(\bhalpha\tp \bq - \balpha_{\sP}\tp \bq)^2}
  &=
    \E_{\sP}\br{(\bhalpha\tp \bq)^2}
    \geq \frac{\Delta^2}{4} \pr{\bx\tp \bM \bx}^2  \sP\pr{|\bhalpha\tp \bq| \geq \frac{\Delta}{2} \bx\tp \bM \bx}\,, \qquad \text{while} \\ 
    \E_{\sQ}\br{(\bhalpha\tp \bq - \balpha_{\sQ}\tp \bq)^2}
  &\geq
    \frac{\Delta^2}{4} \pr{\bx\tp \bM \bx}^2  \sQ\pr{|\balpha_{\sQ}\tp \bq - \bhalpha\tp \bq| \geq \frac{\Delta}{2} \bx\tp \bM \bx}\\
  &\geq \frac{\Delta^2}{4} \pr{\bx\tp \bM \bx}^2  \sQ\pr{|\bhalpha\tp \bq| < \frac{\Delta}{2} \bx\tp \bM \bx}\,,
\end{align*}
where the last inequality comes using the fact that $|a-b| \geq |a| - |b|$ for $a,b \in \R$ and observing that $\balpha_{\sQ}\tp \bq = \bx\tp \bM \bx$.
Summing both inequalities above and applying \cref{lemma:bretagnolle-huber} we get
\begin{align*}
  \E_{\sP}\br{(\bhalpha\tp \bq - \balpha_{\sP}\tp \bq)^2}
  +
  \E_{\sQ}\br{(\bhalpha\tp \bq - \balpha_{\sQ}\tp \bq)^2}
  &\geq
    \frac{\Delta^2}{8} \pr{\bx\tp \bM \bx}^2 \cdot \exp\pr{-\D(\sP,\sQ)}\\
  &\stackrel{(a)}{=}
    \frac{\Delta^2}{8} \pr{\bx\tp \bM \bx}^2 \cdot \exp\pr{-\frac{\Delta^2}{2} \bx\tp \bM \bx}\,,
\end{align*}
where step $(a)$ follows from $\text{KL}$-divergence between multivariate Gaussians with the same covariance matrix.
Now, using a basic fact that $2 \max\cbr{a,b} \geq a + b$, we get that for any measure $\sP$ given by parameter $\balpha$ we have
\begin{align*}
  \E\br{(\bhalpha\tp \bq - \balpha\tp \bq)^2}
  \geq
  \frac{\Delta^2}{16} \pr{\bx\tp \bM \bx}^2 \cdot \exp\pr{-\frac{\Delta^2}{2} \bx\tp \bM \bx}~.
\end{align*}
The statement then follows by choosing $\Delta^2 = (\bx\tp \bM \bx)^{-1}$.
\end{proof}
Now we prove a high-probability version of the just given inequality.
\begin{lemma}
  \label{lem:P_alpha_general_lb}
  For any estimator $\bhalpha$ of $\balpha$ in~\cref{eq:N_alpha_equiv} and any $\delta \in (0,1)$ we have
  \begin{align*}
    \sP\pr{\pr{\bx\tp \bsT \bSigma^{-1}(\bhalpha - \balpha)}^2 \geq \ln\pr{\frac14 \cdot \frac{1}{1 - \delta}} \bx\tp \bM \bx}
    \geq 1 - \delta~.
  \end{align*}
\end{lemma}
\begin{proof}
  The proof is very similar to the proof of~\cref{lem:E_alpha_general_lb} except we will not apply Markov's inequality and focus directly on giving a lower bound the deviation probabilities rather than expectations.
  Thus, similarly as before introduce mean parameters
  $\balpha_{\sP} = \bzero$ and $\balpha_{\sQ} = \Delta (\bPsi\tp\bK^{-1}\bPsi)^{-1} \bq / \pr{\bx\tp \bM \bx}$
  and their associated probability measures
  $\sP = \cN(\bzero, \bK)$ and $\sQ = \cN\pr{\frac{\Delta \bPsi (\bPsi\tp\bK^{-1}\bPsi)^{-1} \bq}{\bx\tp \bM \bx}, \bK}$.

  Note that
  \begin{align*}
    &\sP\pr{|\bhalpha\tp \bq| \geq \frac{\Delta}{2}} = \sP\pr{|\balpha_{\sP}\tp \bq - \bhalpha\tp \bq| \geq \frac{\Delta}{2}}~,\\
    &\sQ\pr{|\balpha_{\sQ}\tp \bq - \bhalpha\tp \bq| \geq \frac{\Delta}{2}}
    \geq
    \sQ\pr{|\bhalpha\tp \bq| < \frac{\Delta}{2}}
  \end{align*}
  and so by using~\cref{lemma:bretagnolle-huber} we obtain an exponential tail bound
  \begin{align*}
    \sP\pr{|\balpha_{\sP}\tp \bq - \bhalpha\tp \bq| \geq \frac{\Delta}{2}}
    +
    \sQ\pr{|\balpha_{\sQ}\tp \bq - \bhalpha\tp \bq| \geq \frac{\Delta}{2}}
    \geq
    \exp(-\D(\sP \,||\, \sQ))
    =
    \frac12 \exp\pr{-\frac{\Delta^2}{\bx\tp \bM \bx}}~.
  \end{align*}
  Setting the r.h.s.\ in the above to $2(1-\delta)$ where $\delta$ is an error probability, and solving for $\Delta$ gives us tuning
  \[
    \Delta^2 = 2 \ln\pr{\frac14 \cdot \frac{1}{1 - \delta}} \bx\tp \bM \bx~.
  \]
  Thus, we get
  \begin{align*}
    \sP\pr{\pr{\balpha_{\sP}\tp \bq - \bhalpha\tp \bq}^2 \geq \ln\pr{\frac14 \cdot \frac{1}{1 - \delta}} \bx\tp \bM \bx}
    +
    \sQ\pr{|\balpha_{\sQ}\tp \bq - \bhalpha\tp \bq| \geq \ln\pr{\frac14 \cdot \frac{1}{1 - \delta}} \bx\tp \bM \bx}
\geq 2 (1-\delta)
  \end{align*}
  and using the fact that $2 \max(a,b) \geq a + b$ we get that for any probability measure $\sP$ given by parameter $\balpha$ we have
  \[
    \sP\pr{\pr{\balpha_{\sP}\tp \bq - \bhalpha\tp \bq}^2 \geq \ln\pr{\frac14 \cdot \frac{1}{1 - \delta}} \bx\tp \bM \bx}
    \geq 1 - \delta~.
  \]
\end{proof}

\section{Proof of the Upper Bounds}\label{appendix:upper-bound-proofs}
\paragraph{\cref{thm:upper_bounds} (restated).}
\emph{  
For the estimator $\bhtheta_n(\bhalpha\MLE)$ and for any $\bx \in \R^d$ we have
\[
  \E\br{\mc{L}(\bx)} = \bx\tp \bM \bx + \bx\tp\bsT\bx + \sigma^2.
\]
Moreover for the same estimator, with probability at least $1 - \delta, \delta \in (0,1)$ we have
\[
  \mc{L}(\bx) \leq 2\ln\pr{\frac{2}{\delta}} \bx\tp \bM \bx
  + \bx\tp\bsT\bx + \sigma^2.
\]
}
\begin{proof}
  Recall that
  \[
    \bhalpha\MLE = (\bPsi\tp\bK^{-1}\bPsi)^{-1}\bPsi\tp\bK^{-1} \bY~.
  \]
  The first result follows from \cref{lem:raw_lower_bound} where we have to give an identity for
  \begin{equation}
    \label{eq:appendix:E_raw_lower_upper}
    \E\br{\pr{\bx\tp\bsT\bSigma^{-1}(\balpha - \bhalpha\MLE)}^2}
  \end{equation}
  and the missing piece is a covariance of the estimator $\bhalpha\MLE$
  \begin{align}
    &\E\br{(\balpha - \bhalpha\MLE)(\balpha - \bhalpha\MLE)\tp} \nonumber\\
    &= (\bPsi\tp\bK^{-1}\bPsi)^{-1}\bPsi\tp\bK^{-1} \Cov(\bY,\bY) \bK^{-1}\bPsi (\bPsi\tp\bK^{-1}\bPsi)^{-1} \nonumber\\
    &= (\bPsi\tp\bK^{-1}\bPsi)^{-1}~. \label{eq:appendix:alpha_MLE_cov}
  \end{align}
  To prove the second result we have to give a high probability upper bound on~\cref{eq:appendix:E_raw_lower_upper}.

  Let $\bq = \bSigma^{-1} \bsT \bx$ and
  observe that $\bq\tp \bhalpha\MLE$ is Gaussian (since $\bY$ is composed of Gaussian entries) with mean $\bq\tp \balpha$ by equivalence~\eqref{eq:N_alpha_equiv}, and covariance $(\bPsi\tp\bK^{-1}\bPsi)^{-1}$ by~\cref{eq:appendix:alpha_MLE_cov}.
  Then, by Gaussian concentration for any error probability $\delta \in (0,1)$ we have
  \begin{align*}
    \sP\pr{(\bq\tp \balpha - \bq\tp \bhalpha\MLE)^2 \geq \sqrt{2 \bq\tp (\bPsi\tp\bK^{-1}\bPsi)^{-1} \bq \ln\pr{\frac{2}{\delta}} }} \leq \delta
  \end{align*}
  which completes the proof.
\end{proof}
\section{Derivation of EM Steps}\label{appendix:m-step}
Recall that our goal is to solve
\[
  \max_{\cE'} \int \ln\pr{p(\bvartheta, \cD \,|\, \cE')} \diff p(\bvartheta \,|\, \cD, \wh{\cE}_t)~.
\]

First, we will focus on the integral.
The chain rule readily gives
\begin{align*}
  \ln p(\bTheta, \cD \,|\, \cE')
  =
  \ln p(\bTheta \,|\, \cD, \cE')
  +
  \ln p(\bTheta \,|\, \cE')\,.
\end{align*}

Using the same reasoning and notation as in the proof of~\cref{prop:posterior_ll} we get
\begin{align*}
  \int \ln p(\bvartheta \,|\, \cD, \cE') \diff p(\bvartheta \,|\, \cD, \wh{\cE}_t)
  &= \sum_{i=1}^n\sum_{j=1}^{m_i}
  \pr{
  \frac{1}{2}\ln\left(\frac{1}{\sigma^2}\right)
  - \frac{1}{2\sigma^2}
    \int (Y_{i,j} - \bx_{i,j}\tp\bvartheta_i)^2 \diff p(\bvartheta_i \,|\, \cD, \wh{\cE}_t)
  }
  + \const(\cE')\\
  &= \sum_{i=1}^n\sum_{j=1}^{m_i}
  \pr{
  \frac{1}{2}\ln\left(\frac{1}{\sigma^2}\right)
  - \frac{1}{2\sigma^2}
  (Y_{i,j} -\bx_{i,j}\tp\bmu_i)^2 - \bx_{i,j}\tp\bsT_i\bx_{i,j}
  }
  + \const(\cE')\,,
\end{align*}
using the fact that
$
\int (Y_{i,j} - \bx_{i,j}\tp\bvartheta_i)^2 \diff p(\bvartheta_i \,|\, \cD, \wh{\cE}_t)
= (Y_{i,j} -\bx_{i,j}\tp\bmu_i)^2
+ \bx_{i,j}\tp\bsT_i\bx_{i,j}
$
where we took $\btheta_i \sim \cN(\bmu_i,\bsT_i)$ according to~\cref{prop:posterior_ll}.

Now we compute the expected log-likelihood of the vector of task parameters:
\begin{align*}
  \int \ln p(\bvartheta \,|\, \cE') \diff p(\bvartheta \,|\, \cD, \wh{\cE}_t)
  = \frac{n}{2} \ln\det\bSigma^{-1}
  -  \frac12 \sum_{i=1}^n \int (\bvartheta_i - \balpha)\tp\bSigma^{-1} (\bvartheta_i - \balpha) \diff p(\bvartheta_i \,|\, \cD, \wh{\cE}_t)
  + \const(\cE')~.
\end{align*}

\paragraph{M-step for $\sigma^2$.} Now, note that since the likelihood of the vector of task variables $\bTheta$ does not depend on the parameter $\sigma^2$ we can solve for $\sigma^2$ based on the first order condition of the problem above. Differentiating the above equation with respect to $\sigma^{-2}$ (and ignoring the constant) gives
\begin{equation}
\sum_{i=1}^n\sum_{j=1}^{m_i}\left(\sigma^2
- \pr{(Y_{i,j} -\bx_{i,j}\tp\bmu_i)^2 + \bx_{i,j}\tp\bsT_i\bx_{i,j}}\right).
\end{equation}
while setting the derivative to zero gives
\begin{equation}
\sigma^2 =
\frac{1}{n}\sum_{i=1}^n\frac{1}{m_i}\sum_{j=1}^{m_i}
\left((Y_{i,j} - \bx_{i,j}\tp\bmu_i)^2 + \bx_{i,j}\tp\bsT_i\bx_{i,j}
\right).
\end{equation}

\paragraph{M-step for $\balpha$.}
Differentiating the objective w.r.t.\ $\balpha$ (and ignoring the constant) gives
$
\sum_{i=1}^n\bSigma^{-1}(\E[\btheta_i] - \balpha)
$
from which we get
\begin{equation}
\balpha = \sum_{i=1}^n\bmu_i~.
\end{equation}
\paragraph{M-step for $\bSigma$.}
Differentiating the expected log-likelihood of the vector
of task parameters with respect to $\bA = \bSigma^{-1}$ gives
\begin{equation}
  \sum_{i=1}^n
  \tr(\bSigma d\bA) - \tr
  \int
  \pr{(\bvartheta_i - \balpha)(\bvartheta_i - \balpha)\tp d\bA}
  \diff p(\bvartheta_i \,|\, \cD, \wh{\cE}_t)
\end{equation}
from which we get
\begin{equation}
\bSigma = \frac{1}{n}\sum_{i=1}^n
\E[(\btheta_i - \balpha)(\btheta_i - \balpha)\tp].
\end{equation}
Finally, computing the expectation
\begin{equation}
\sum_{i=1}^n\left(
\E[\btheta_i\btheta_i\tp] - 2\bmu_i\balpha\tp + \balpha\balpha\tp\right)
= \sum_{i=1}^n\left((\bmu_i - \balpha)(\bmu_i - \balpha)\tp
+ \bsT_i\right)
\end{equation}
shows the update for $\bSigma$.

\section{Selecting $\lambda$ in Biased Regression}
\label{sec:lambdaopt}
The parameter $\lambda$ is selected via random search in the following
way. For each of the 50 samples of $\lambda$ from log-uniform
distribution on interval $[0;100]$ we perform the following procedure
to estimate the risk $\wh{L}$. Firstly, we split the training tasks
into $K=10$ groups $\mc{S}_1,\dots,\mc{S}_K$ of (approximately)
equal size and compute the estimates $\bhalpha_k$ using the data
$\mc{S}^{\setminus k}$ from all of the groups excluding the group
$k$: $\mc{S}^{\setminus k} := \cup_{i\neq k}\mc{S}_i$. For each of
the estimated values $\bhalpha_k$ we perform adaptation to and
testing on the tasks in the group $\mc{S}_k$ using the given value
of $\lambda$. We split the samples of each task data $D_i\in\mc{S}_k$
randomly into adaptation and test sets $10$ times each time such
that the size of adaptation set is close to the size of adaptation
sets used with the actual test data. For each of the splits we
compute an estimate of the parameter vector $\wh\btheta_{k,i,l}$
where $k$ is the index of the group which was not used to estimate
$\wh\balpha_k$, $i$ is the index of a task data $D_i\in\mc{S}_k$,
$l$ is the index of a random split of the samples in that task into
adaptation and test sets.  With this parameter vector and using the
test set of the task $D_i\in\mc{S}_k$ we can also estimate the loss
$\wh{L}_{k,i,l}$ after which all the loss values are averaged:
\begin{equation*}
\wh{L} = \frac{1}{K}\sum_{k=1}^K\frac{1}{|\mc{S}^{\setminus k}|}
\sum_{i : D_i\in\mc{S}^{\setminus k}}\frac{1}{10}\sum_{l=1}^{10}
\wh{L}_{k,i,l}.
\end{equation*}
At the end we select the value of $\lambda$ which lead to the
smallest value of $\wh{L}$ using this cross-validation procedure.

%

\section{Supplementary Statements}
\label{appendix:supp_statements}

\begin{proposition}
  \label{prop:unrolling_lb}
  For $\bM$ (see \cref{eq:unbiased-lower-bound}) we have
  \begin{align*}
    \bM
    =
    \sigma^4 \cdot \pr{\bSigma \bX_n\tp \bX_n + \sigma^2 \bI}^{-1}
    \bA^{-1}
    \pr{\bSigma \bX_n\tp \bX_n + \sigma^2 \bI}^{-1}\,,
  \end{align*}
  where we denote
  \[
    \bA = \sum_{i=1}^n \bX_i\tp (\bX_i \bSigma \bX_i\tp + \sigma^2 \bI)^{-1} \bX_i~.
  \]
\end{proposition}
\begin{proof}
  Recall that
  \[
    \bM = \bsT \bSigma^{-1} \pr{\bPsi\tp \bK^{-1} \bPsi}^{-1} \bSigma^{-1} \bsT
  \]
and observe that
\begin{align*}
  \bK^{-1}
  &= \begin{bmatrix}
    (\bX_1 \bSigma \bX_1\tp + \sigma^2 \bI)^{-1} & \bzero & \dots & \bzero\\
    \bzero & (\bX_2 \bSigma \bX_2\tp + \sigma^2 \bI)^{-1} & \dots & \bzero\\
    \vdots & & \ddots & \vdots\\
    \bzero & \bzero & \dots & (\bX_n \bSigma \bX_n\tp + \sigma^2 \bI)^{-1}
  \end{bmatrix}\,,
\end{align*}
which in turn implies
\begin{align*}
  \bPsi\tp \bK^{-1} \bPsi
  =
  \sum_{i=1}^n \bX_i\tp (\bX_i \bSigma \bX_i\tp + \sigma^2 \bI)^{-1} \bX_i~.
\end{align*}
On the other hand,
\begin{align*}
  \bsT \bSigma^{-1}
  &= \pr{\bSigma^{-1} + \frac{1}{\sigma^2} \bX_n\tp \bX_n}^{-1} \bSigma^{-1}\\
  &= \sigma^2 \pr{\sigma^2 \bI + \bSigma \bX_n\tp \bX_n}^{-1}~.
\end{align*}
Combining the above gives the statement.
\end{proof}

\begin{lemma}
  \label{lem:lambda_Sigma_Tau}
  In the following assume that $\bX_i\tp \bX_i = \frac{m_i}{d} \bI$ for all $i.$  
  Let $\lambda_j(\bSigma)$ be the $j$th eigenvalue of $\bSigma$.
  Then,
  \begin{align*}
    \lambda_j(\bM) = \sigma^4 \cdot \frac{d^2}{\pr{m_n \lambda_j(\bSigma) + d \sigma^2}^2} \cdot \frac{\text{HM}\pr{\lambda_j(\bSigma) + \frac{d \sigma^2}{m_i}}_{i=1}^n}{n}\,,
  \end{align*}
  where $\text{HM}( z_i )_{i=1}^n$ denotes the harmonic mean of sequence $(z_i)_{i=1}^n$.
  Moreover,
  \begin{align*}
    \lambda_j(\bsT) = \frac{d \sigma^2 \lambda_j(\bSigma)}{d \sigma^2 + m_n \lambda_j(\bSigma)}~.
  \end{align*}
  
  Finally, the eigenvectors of $\bM$ and $\bsT$ coincide with the eigenvectors of $\bSigma$.
\end{lemma}
\begin{proof}
  We first characterize eigenvalues of matrix $\bM$.
By~\cref{prop:unrolling_lb},
  \[
    \bM
  =
  \sigma^4 \cdot \pr{\bSigma \bX_n\tp \bX_n + \sigma^2 \bI}^{-1}
  \bA^{-1}
  \pr{\bSigma \bX_n\tp \bX_n + \sigma^2 \bI}^{-1}~.
\]
We start with $\bA^{-1}$, and by the spectral theorem, $\bSigma = \bU \bLambda \bU\tp$ for some unitary $\bU$ and diagonal $\bLambda$:
\begin{align*}
  \bA^{-1} = \pr{\sum_{i=1}^n \bX_i\tp (\bX_i \bSigma \bX_i\tp + \sigma^2 \bI)^{-1} \bX_i}^{-1}
  &=
    \pr{\sum_{i=1}^n (\bSigma \bX_i\tp \bX_i + \sigma^2 \bI)^{-1} \bX_i\tp \bX_i}^{-1}\\
  &=
    \pr{\sum_{i=1}^n \pr{\bSigma \cdot \frac{m_i}{d} + \sigma^2 \bI}^{-1} \frac{m_i}{d}}^{-1}\\
  &=
    \pr{\sum_{i=1}^n \pr{\bU \bLambda \bU\tp \cdot \frac{m_i}{d} + \sigma^2 \bI}^{-1} \frac{m_i}{d}}^{-1}\\
  &=
    \bU \pr{\sum_{i=1}^n \pr{\bLambda + \frac{d \sigma^2}{m_i} \cdot \bI}^{-1}}^{-1}  \bU\tp~.
\end{align*}
Now,
\begin{align*}
  \pr{\bSigma \bX_n\tp \bX_n + \sigma^2 \bI}^{-1}
  &=
    \pr{\bSigma \cdot \frac{m_n}{d} + \sigma^2 \bI}^{-1}\\
  &=
    \pr{\bU \bLambda \bU\tp \cdot \frac{m_n}{d} + \sigma^2 \bI}^{-1}\\
  &=
    \bU \pr{\bLambda  \cdot \frac{m_n}{d} + \sigma^2 \bI}^{-1} \bU\tp.
\end{align*}
Thus,
\begin{align*}
  \bM = \bU \pr{\pr{\bLambda  \cdot \frac{m_n}{d} + \sigma^2 \bI}^2 \sum_{i=1}^n \pr{\bLambda + \frac{d \sigma^2}{m_i}}^{-1}}^{-1} \bU\tp~.
\end{align*}
and moreover the $j$th eigenvalue of $\bM$ is
\begin{align*}
  \lambda_j(\bM)
  &= \frac{1}{\pr{\frac{m_n}{d} \lambda_j(\bSigma) + \sigma^2}^2} \cdot \frac{1}{\sum_{i=1}^n \frac{1}{\lambda_j(\bSigma) + \frac{d \sigma^2}{m_i}}}\\  
  &= \frac{1}{\pr{\frac{m_n}{d} \lambda_j(\bSigma) + \sigma^2}^2} \cdot \frac{\text{HM}\pr{\lambda_j(\bSigma) + \frac{d \sigma^2}{m_i}}_{i=1}^n}{n}\,.
\end{align*}
where recall that $\text{HM}(z_i)_{i=1}^n$ denotes the harmonic mean of sequence $(z_i)_{i=1}^n$.

Using the same arguments as above
\begin{align*}
  \bsT
  &= \pr{\bSigma^{-1} + \frac{1}{\sigma^2} \bX_n\tp \bX_n}^{-1}\\
  &= \pr{\bU \bLambda^{-1} \bU\tp + \frac{m_n}{d \sigma^2}}^{-1}
\end{align*}
and so
\begin{align*}
  \lambda_j(\bsT) = \frac{1}{\frac{1}{\lambda_j(\bSigma)} + \frac{m_n}{d \sigma^2}} = \frac{d \sigma^2 \lambda_j(\bSigma)}{d \sigma^2 + m_n \lambda_j(\bSigma)}~.
\end{align*}
Finally, in both cases of $\bM$ and $\bsT$ we observe that their eigenvectors are eigenvectors of $\bSigma$.
\end{proof}

\paragraph{\cref{prop:unrolled_lb_Sigma_cases,cor:unrolled_lb_Sigma_cases2} (restated).}
\emph{In the following assume that $\bX_i\tp \bX_i = \frac{m_i}{d} \bI$ for all $i.$
  For $\bSigma = \tau^2 \bI$, any $\bx \in \R^d$, and any $c > 0$,}
\begin{align*}
  c \bx\tp \bM \bx
  +
  \bx\tp \bsT \bx
  +
  \sigma^2
  =
  c \cdot \frac{H_{\tau^2}}{n} \cdot \frac{d^2 \sigma^4}{\pr{\tau^2 m_n + d \sigma^2}^2} \cdot \|\bx\|^2 +
  \frac{d \sigma^2 \tau^2}{\tau^2 m_n + d \sigma^2} \cdot \|\bx\|^2
  +
  \sigma^2~,
\end{align*}
\emph{where $H_{\tau^2}$ is a harmonic mean of the sequence $\pr{\tau^2 + \frac{d \sigma^2}{m_i}}_{i=1}^n$.}

\emph{Moreover, let $\bSigma$ be a \ac{PSD} matrix of rank $s \leq d$ with eigenvalues $\lambda_1 \geq \ldots \geq \lambda_s > 0$.
  Then for any $\bx \in \R^d$ and any $c > 0$,}
\begin{align*}
  c \bx\tp \bM \bx
  +
  \bx\tp \bsT \bx
  +
  \sigma^2
  \geq
  c \cdot \frac{H_{\lambda_s}}{n} \cdot \frac{d^2 \sigma^4}{\pr{\lambda_1 m_n + d \sigma^2}^2} \cdot \|\bx\|_{\bP_s\tp \bP_s}^2 +
  \frac{d \sigma^2 \lambda_s}{\lambda_s m_n + d \sigma^2} \cdot \|\bx\|_{\bP_s\tp \bP_s}^2
  +
  \sigma^2
\end{align*}
where $\bP_s = [\bu_1, \ldots, \bu_s]\tp$ and $(\bu_j)_{j=1}^s$ are eigenvectors of $\bSigma$.
\begin{proof}
  Recalling that by~\cref{prop:unrolling_lb},
  \[
    \bM
  =
  \sigma^4 \cdot \pr{\bSigma \bX_n\tp \bX_n + \sigma^2 \bI}^{-1}
  \bA^{-1}
  \pr{\bSigma \bX_n\tp \bX_n + \sigma^2 \bI}^{-1}~.
\]
and using \cref{lem:lambda_Sigma_Tau} with $\bSigma = \tau^2 \bI$ we get the first result.

Now we turn to the low-rank case.
We start by considering a \ac{PSD} matrix $\bSigma_{\ve}$ with $s$ eigenvalues $\lambda_1 \geq \ldots \geq \lambda_s > 0$ and remaining $d-s$ are $\ve > 0$.
Denote also by $\bM_{\ve}$, $\bsT_{\ve}$ matrices w.r.t. $\bSigma_{\ve}$.
The idea is to lower bound $\bx\tp \bM_{\ve} \bx$ and $\bx\tp \bsT_{\ve} \bx$ and then analyze a limiting behavior as $\ve \to 0$.

By \cref{lem:lambda_Sigma_Tau}, $\bM_{\ve}$, $\bsT_{\ve}$, and $\bSigma_{\ve}$ share the same eigenvectors $\bu_1, \ldots, \bu_s$, and so
\begin{align*}
  c \bx\tp \bM_{\ve} \bx
  +
  \bx\tp \bsT_{\ve} \bx
  &=  
  c \sum_{j=1}^d \pr{\bu_j\tp \bx}^2 \lambda_j(\bM_{\ve})
  +
  \sum_{j=1}^d \pr{\bu_j\tp \bx}^2 \lambda_j(\bsT_{\ve})\\
  &=
    c \cdot \sum_{j=1}^s \frac{H_{\lambda_j}}{n} \cdot \frac{\sigma^4}{\pr{\lambda_j \frac{m_n}{d} + \sigma^2}^2} \pr{\bu_j\tp \bx}^2
    +
    c \cdot \underbrace{\frac{H_{\ve}}{n} \cdot \frac{\sigma^4}{\pr{\ve \frac{m_n}{d} + \sigma^2}^2}}_{(a)} \pr{ \sum_{j=s+1}^d \pr{\bu_j\tp \bx}^2 }\\
  &+
    \sum_{j=1}^s \frac{\sigma^2 \lambda_j}{\lambda_j \frac{m_n}{d} + \sigma^2} \pr{\bu_j\tp \bx}^2
    + \frac{\sigma^2 \ve}{\ve \frac{m_n}{d} + \sigma^2} \pr{ \sum_{j=s+1}^d \pr{\bu_j\tp \bx}^2 }~.
\end{align*}
Now,
\begin{align*}
  &\lim_{\ve \to 0}\pr{c \bx\tp \bM_{\ve} \bx
  +
  \bx\tp \bsT_{\ve} \bx}\\
  &=
    c \cdot \sum_{j=1}^s \frac{H_{\lambda_j}}{n} \cdot \frac{\sigma^4}{\pr{\lambda_j \frac{m_n}{d} + \sigma^2}^2} \pr{\bu_j\tp \bx}^2
    +
    \frac{d \sigma^2}{M} \sum_{j=s+1}^d \pr{\bu_j\tp \bx}^2
  +
    \sum_{j=1}^s \frac{\sigma^2 \lambda_j}{\lambda_j \frac{m_n}{d} + \sigma^2} \pr{\bu_j\tp \bx}^2\\
  &\geq
    c \cdot \frac{H_{\lambda_s}}{n} \cdot \frac{\sigma^4}{\pr{\lambda_1 \frac{m_n}{d} + \sigma^2}^2} \sum_{j=1}^s \pr{\bu_j\tp \bx}^2        
    +
    \frac{\sigma^2 \lambda_s}{\lambda_s \frac{m_n}{d} + \sigma^2} \sum_{j=1}^s \pr{\bu_j\tp \bx}^2\,,
\end{align*}

where we note that the limit of term $(a)$ is handled as
\begin{align*}
  \lim_{\ve \to 0}\frac{1}{\sum_{i=1}^n \frac{1}{\ve + \frac{d \sigma^2}{m_i}}} \cdot \frac{\sigma^4}{\pr{\ve \frac{m_n}{d} + \sigma^2}^2} = \frac{d \sigma^2}{M} \geq 0~.
\end{align*}

\end{proof}

\section{Further Experimental Details and Results}
\label{sec:apxexps}
In this section we provide extra figures for our experimental results.
\cref{fig:fourier-examples} is complemented with
\cref{fig:fourier-examples2},  adding a second example in addition to the one shown in the previous figure.

\begin{figure*}
\centering
\includegraphics[width=0.31\linewidth]{fourier/example-task-0-ntasks-10.png}
\includegraphics[width=0.31\linewidth]{fourier/example-task-0-ntasks-50.png}
\includegraphics[width=0.31\linewidth]{fourier/example-task-0-ntasks-100.png}
\includegraphics[width=0.31\linewidth]{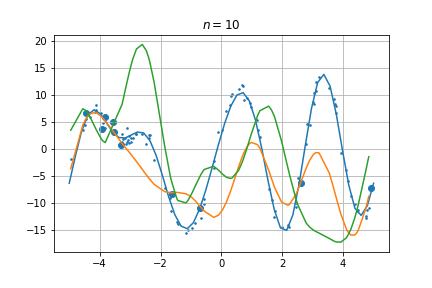}
\includegraphics[width=0.31\linewidth]{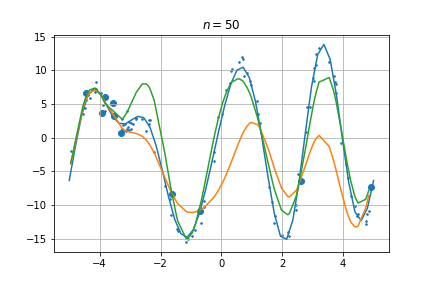}
\includegraphics[width=0.31\linewidth]{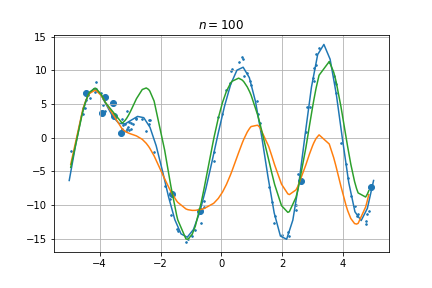}
\caption{
Two sets of examples of predictions on the synthetic,  `Fourier' meta-learning problem.
Top and bottom rows correspond to different (random) instances; the top row
in fact replicates \cref{fig:fourier-examples}.
Training data is shown in bold, small dots show test data.
We also show the predictions for two learners (at every input) and the target function.
The column correspond to outputs
obtained training on $n\in\{10,50,100\}$ tasks.
}
\label{fig:fourier-examples2}
\end{figure*}

For completeness, the pseudocode of \ac{MoM} is given in \cref{alg:mom}.
%
\begin{algorithm}[H]
\caption{\ac{MoM} Estimator for Learning Linear Features of \citep{tripuraneni2020provable}}\label{alg:mom}
\begin{algorithmic}
  \REQUIRE $\pr{(\bx_{1,j}, y_{1,j})}_{j=1}^{m_1}, \ldots, \pr{(\bx_{m_{n-1},j}, y_{m_{n-1},j})}_{j=1}^{m_{i-1}}$ --- training examples from $n-1$ past tasks, $s$ --- problem rank.
  \STATE $\bU \bD \bV\tp \gets \text{SVD}\pr{\frac{1}{M-m_n} \sum_{i=1}^{n-1} \sum_{j=1}^{m_n} y_{i,j}^2 \bx_{i,j}\bx_{i,j}\tp}$
  \STATE $\bhB \gets [D_{1,1} \bu_1, \ldots, D_{s,s} \bu_s]$
\STATE \textbf{return} $\bhB$
\end{algorithmic}
\end{algorithm}

\end{document}